\documentclass{article}
\usepackage{arxiv}

\usepackage{graphicx}%
\usepackage[utf8]{inputenc} 
\usepackage[T1]{fontenc}
\usepackage{multirow}%
\usepackage{amsmath,amssymb,amsfonts}%
\usepackage{amsthm}%
\usepackage{mathrsfs}%
\usepackage[title]{appendix}%
\usepackage{xcolor}%
\usepackage{textcomp}%
\usepackage{manyfoot}%
\usepackage{booktabs}%
\usepackage{algorithm}%
\usepackage{algorithmicx}%
\usepackage{algpseudocode}%
\usepackage{listings}%
\usepackage{subcaption}
\usepackage[numbers]{natbib}
\usepackage{array}
\usepackage{hyperref}

\theoremstyle{thmstyleone}%
\newtheorem{theorem}{Theorem}
\newcommand{\added}[1]{#1}
\theoremstyle{thmstyletwo}%
\newtheorem{example}{Example}%

\theoremstyle{thmstylethree}%
\newtheorem{definition}{Definition}%

\DeclareMathOperator*{\argmin}{arg\,min}
\renewcommand{\vec}[1]{\boldsymbol{#1}}
\newcommand{\prob}{\mathbb{P}}
\newcommand{\given}{\, | \,}

\title{A calibration test for evaluating set-based epistemic uncertainty representations}

\author{Mira Jürgens \\
	Department of Data Analysis \\ and Mathematical Modeling\\
	Ghent University\\
	\texttt{mira.juergens@ugent.be} \\
	\And
    Thomas Mortier \\ 
	Department of Environment\\
	Ghent University\\
	\texttt{thomasf.mortier@ugent.be} \\
    \And
    Eyke Hüllermeier \\
    Department of Informatics\\
    Munich Center for Machine Learning \\
    LMU Munich \\
    \texttt{eyke@ifi.lmu.de} \\
	\And
    Viktor Bengs \\
    Department of Informatics\\
    Munich Center for Machine Learning \\
    LMU Munich \\
    \texttt{viktor.bengs@ifi.lmu.de} \\
    \And
    Willem Waegeman\\ 
	Department of Data Analysis \\ and Mathematical Modeling\\
	Ghent University\\
	\texttt{willem.waegeman@ugent.be} \\
}

\begin{document}
\maketitle
\begin{abstract}
The accurate representation of epistemic uncertainty is a challenging yet essential task in machine learning. A widely used representation corresponds to convex sets of probabilistic predictors, also known as credal sets. One popular way of constructing these credal sets is via ensembling or specialized supervised learning methods, where the epistemic uncertainty can be quantified through measures such as the set size or the disagreement among  members. In principle, these sets should contain the true data-generating distribution. As a necessary condition for this validity, we adopt the strongest notion of calibration as a proxy. Concretely, we propose a novel statistical test to determine whether there is a convex combination of the set’s predictions that is calibrated in distribution. In contrast to previous methods, our framework allows the convex combination to be instance dependent, recognizing that different ensemble members may be better calibrated in different regions of the input space. Moreover, we learn this combination via proper scoring rules, which inherently optimize for calibration. Building on differentiable, kernel-based estimators of calibration errors, we introduce a nonparametric testing procedure and demonstrate the benefits of capturing instance-level variability on of synthetic and real-world experiments.
\end{abstract}
%
\section{Introduction}\label{sec: introduction}
In supervised machine learning, it has become more and more important not only to have accurate predictors,
but also to provide a reliable quantification of predictive uncertainty, i.e., the learner's uncertainty in the outcome $y \in \mathcal{Y}$ given a query instance $x \in \mathcal{X}$ for which a prediction is sought. 
Predictive uncertainty is often 
divided into \textit{aleatoric} and \textit{epistemic} uncertainty \cite{SENGE201416,hullermeierAleatoricEpistemicUncertainty2019,kendallWhatUncertaintiesWe2017,gruber2023sources}, where the former corresponds to uncertainty that cannot be reduced with further information (e.g.\, more training data), as it originates from inherent randomness in the relationship between features $X$ and labels $Y$. 
Therefore, we assume that the \textit{ground truth} is a conditional probability distribution $\prob_{Y \given X}$ on $\mathcal{Y}$,
 i.e.\ given an input sample $x \in \mathcal{X}$, each outcome $y$ has a certain probability to occur,
 given by $\prob_{Y \given X=x}$.
 Even with perfect knowledge about the underlying data-generating process, the outcome cannot be predicted with certainty. However, in a typical machine learning scenario, the learner does not know $\prob_{Y|X}$. Having a space of possible hypotheses, an estimator of the underlying probability distribution $f$ within this space typically consists of a mapping $f: \mathcal{X} \rightarrow \mathbb{P}(\mathcal{Y})$, where $\mathcal{X}$ denotes the feature space and $\mathbb{P}(\mathcal{Y})$ the space of all probability distributions over the target space $\mathcal{Y}$. In essence, epistemic uncertainty refers to the uncertainty about the true $\prob_{Y \given X}$, or the ``gap''  between $\prob_{Y \given X}$ and $f$.
One approach to represent this gap is 
via \textit{second-order probability distributions}, assigning a probability for each of the first-order predicted probability distributions, i.e.\ the candidates for $f$. This is commonly done in Bayesian methods, such as Gaussian processes and Bayesian neural networks \cite{gelmanbda04}, but also in evidential deep learning methods \cite{ulmer2024priorposteriornetworks}. 
However, the former usually involves computationally costly methods to approximate the intractable posterior distribution, while the latter has been criticised for producing unfaithful or unreliable representations of epistemic uncertainty \cite{bengs2022neurips,bengs2023icml,meinert2023unreasonable,juergens2024is}.
An alternative way of representing epistemic uncertainty -- which will be the topic of this paper -- is through \textit{sets} of probability distributions. Such sets are often referred to as \textit{credal sets} \cite{walley1991} in the imprecise probability literature,
and are commonly assumed to be closed and convex \cite{cozman2000credal}. In essence, they are designed 
to represent ignorance, i.e.\ a lack of knowledge about the underlying ground truth, by not committing to one, 
but a set of potential probability distributions.
They can be obtained in a direct manner, as via credal classifiers \cite{hullermeier2022credalunc,javanmardi2024conformalized,wang2025creinns,caprio2024credal,nguyen2022measure},
 which have recently gained popularity, or in an indirect manner, via various types of ensemble methods, such as bootstrapped ensembles, deep ensembles \cite{NIPS2017_lakshminarayan} 
 and randomization techniques based on Monte Carlo dropout \cite{gal2016dropout}.
  In credal sets, epistemic uncertainty is quantified via the size of the credal set  \cite{sale2023volume} 
  or the diversity of the corresponding ensemble \cite{NIPS2017_lakshminarayan}.

Due to a lack of an objective ground-truth, epistemic uncertainty representations are often evaluated in an indirect manner, using downstream tasks such as out-of-distribution detection \cite{ovadia2019nips}, robustness to adversarial attacks \cite{kopetzki2021evaluating}, and active learning \cite{nguyen2022measure}. However, recent studies have raised concerns about the usefulness of such tasks w.r.t.\ epistemic uncertainty evaluation \cite{abe2022deep,meinert2023unreasonable,bengs2022neurips}. A crucial open question is whether existing representations of epistemic uncertainty can be interpreted in a statistically profound way. For the case of credal sets, one might wonder whether such representations are statistically \textit{valid}, i.e., whether a credal set contains with high probability the true underlying conditional target distribution. \added{One way to evaluate this validity criterion is via statistical tests, as exemplary visualized in Figure \ref{fig: setting hypothesis testing}.}  Chao et al.\ \cite{chau2024credal} address this issue by proposing a two-sample test framework for credal sets. However, they assume having a sample of sufficient size from the same underlying conditional distribution. In the "classical" machine learning scenario that we look at, we do not have access to the ground truth conditional distribution nor to more than one realization $(x_i,y_i)$ from it, making this direct evaluation infeasible.

One necessary condition for validity is \textit{calibration} \cite{niculescu2005predicting}, which measures the consistency between predicted probabilities and actual frequencies. In this paper we will utilize in this paper the notion of distribution calibration \cite{vaicenavicius19a} as a surrogate method for  validity. Failing calibration necessarily means failing validity, hence in this case one can be confident that the true data generating distribution is not contained inside the credal set. 
Hence, by introducing a calibration test for set-based epistemic uncertainty representations, we aim for a more direct evaluation than measuring the performance on downstream tasks. 

  Our work builds further upon the work of Mortier et al.\ \cite{mortier2023calibration}, who first proposed the usage of calibration as a proxy for epistemic uncertainty evaluation of credal sets. The test introduced in that study analysed whether, for a given convex set of probabilistic models, there exists a calibrated convex combination in the set. However, this existing test has important limitations, which motivate our current work. Specifically, the original test focuses on determining whether a \textit{single} convex combination of ensemble predictions can achieve calibration, but it does so in an instance-agnostic manner. This limitation means that the approach does not account for the variability of model calibration in dependence on the instance space. Moreover,
  the previous test tries to simulate the distribution of the calibration error estimators under the null hypothesis via sampling predictions in the credal set at random, which possibly leads to an overestimation of how calibrated the found convex combination is.
   In this paper, we address these shortcomings by introducing an instance-dependent approach to calibration testing, with two important modifications: First, due to the instance-dependency of the underlying convex combination of probabilistic predictions, we consider situations where predictions of the set are "differently well" calibrated for different regions in the instance space. Second, by using an optimization-based instead of a sampling-based algorithm, we achieve a more reliable estimate of the calibration statistic under the null hypothesis. Together the two modifications result in better control of the statistical Type I error, while simultaneously increasing the power of the test.


The paper is organized as follows. In Section~2 we review the literature on calibration, and we formally discuss the calibration measures that are used further on. In Section~3 we introduce both the concept 
of instance-dependent validity as well as the (weaker) notion of calibration for credal sets.  We then propose
our novel, nonparametric calibration test that is able to assess whether an epistemic uncertainty representation via credal sets is calibrated. Furthermore, we introduce an optimization algorithm for finding the most calibrated convex combination. It includes training a neural network as a meta-learner, using proper scoring rules as loss functions. In Section 4 we empirically evaluate the statistical Type I and Type II error of our test in different scenarios, thereby showing its improvement over the test proposed by Mortier et al.\ \cite{mortier2023calibration} We further demonstrate its usefulness on validating common ensemble-based epistemic uncertainty representations of models trained on real-world datasets.

\section{Calibration of first-order predictors}
 \label{sec: calibration errors}
\begin{figure}[t]
    \centering
    \includegraphics[width=0.8\linewidth]{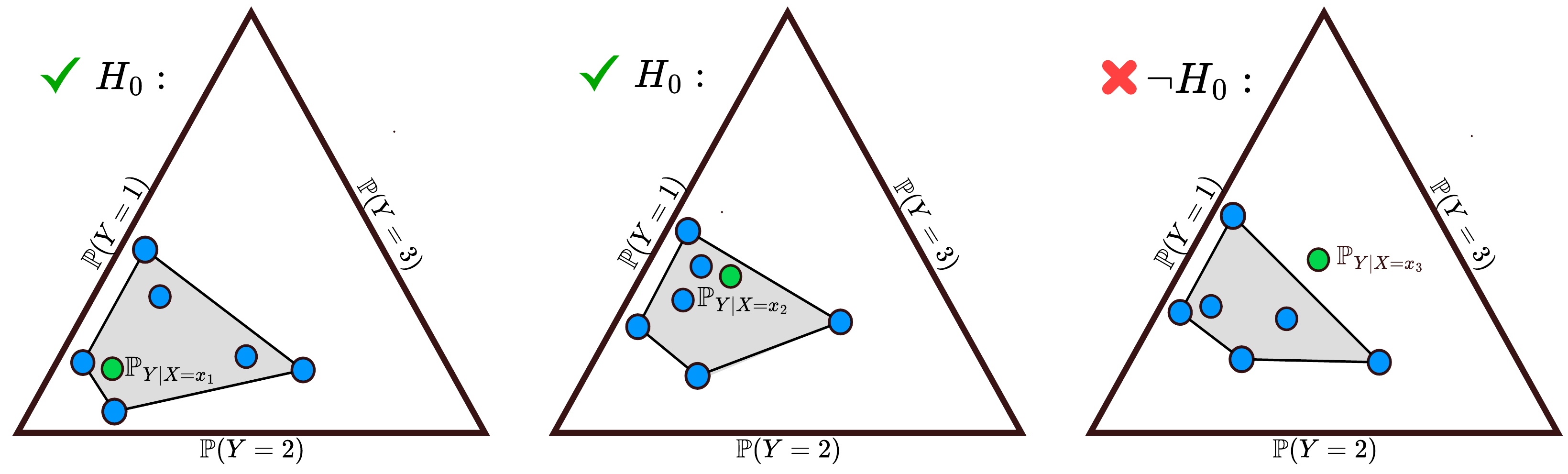}
    \caption{\added{Exemplary illustration of a desired statistical test of validity for the
  the case of $K=3$ classes. Ideally, it should be \textit{instance-dependent}, that is, allow the ground truth to be a convex combination with weights that vary across the instance space.   In the left and middle panel the credal sets (composed by the blue dots)
  is \emph{valid} because the true conditional distribution (green dot) lies
  inside the convex hull. In the right panel the true conditional distribution lies outside the convex hull, hence the null hypothesis should be \textit{rejected}. In this work, we use calibration as an approximation to test validity, that is, we replace $\mathbb{P}_{Y|X}$ with $\mathbb{P}_{Y|f(X)}$ for a probabilistic predictor $f$.}}
    \label{fig: setting hypothesis testing}
\end{figure}
We start with reviewing common calibration metrics that have been introduced for evaluating distribution calibration in multi-class classification. We also discuss differentiable estimators of these metrics, and statistical tests that assess calibration for a \textit{single} probabilistic model. The extension to \textit{sets} of probabilistic models will be discussed in Section~3. 

\subsection{General setting}
We assume a multi-class classification setting with feature space $\mathcal{X} \subseteq \mathbb{R}^d$ and label space $\mathcal{Y}=\{1, \dots, K\}$, consisting of $K$ classes. Let $X$ and $Y$ be random variables that are distributed according to a joint distribution $\mathbb{P}_{X,Y}$ on $\mathcal{X} \times \mathcal{Y}$. A probabilistic model can be represented as $f: \mathcal{X} \rightarrow \mathbb{P}(\mathcal{Y})$, which is a map from the feature space to the space of probability distributions over the output space $\mathcal{Y}$. For multi-class classification problems with $K$ classes, $\mathbb{P}(\mathcal{Y}) = \Delta_K$ (the $(K-1)$-dimensional probability simplex.
Roughly speaking, a classifier is \textit{calibrated} if its outputs coincide with probability distributions that match the empirical frequencies observed from realized outcomes. In multi-class classification, one distinguishes between different types of calibration. \textit{Confidence calibration} \cite{niculescu2005predicting,guo2017,pmlr-v80-kumar18a} only analyses the confidence score, i.e., the probability of the predicted class being calibrated. It is therefore the weakest notion of calibration. \textit{Distribution calibration} \cite{brockerReliabilitySufficiencyDecomposition2009,vaicenavicius19a}, which will be the focus in this paper, analyses whether all class probabilities are calibrated. Instead of only requiring calibrated marginal probabilities as in the definition of \textit{classwise calibration} \cite{zadrozny2002transforming}, it is defined via conditioning on the full probability vector; hence, it is the strongest notion of calibration.

\begin{definition}
\label{def: strong calibration}
A probabilistic multi-class classifier with output vector $f(X) \in \Delta_K$ is \textit{calibrated in distribution}
 if for all $k \in \mathcal{Y}=\{1, \dots, K\}$ it holds that
    $$\mathbb{P}(Y = k | f(X) = \vec{s}) = s_k,$$
with $s_k$ being the $k$-th component of the probability vector $\vec{s} \in \Delta_K$.
\end{definition}

To evaluate calibration, one typically makes use of calibration \textit{errors}. 
In general, these errors measure the (average) discrepancy between the conditional probability distribution $\mathbb{P}_{Y|f(X)}$ and the predictions $f(X)$. Different calibration errors have been introduced to measure distribution calibration. One can in general differentiate between calibration errors defined as an \textit{expectation} over a divergence $d: \Delta_K \times \Delta_K \rightarrow [0, \infty)$ between $f(X)$ and $\mathbb{P}_{Y|f(X)}$ \cite{popordanoska2022consistent,pmlr-v238-popordanoska24a}: \begin{equation}
\label{eq: calibration as expectation}
         \mathrm{CE}_d(f) = \mathbb{E}\Big[d\big(f(X), \mathbb{P}_{Y|f(X)}\big)\Big],
    \end{equation}
and calibration errors defined via \textit{integral probability metrics} \cite{muller1997integral}, also called \textit{kernel calibration errors} \cite{pmlr-v80-kumar18a,widmann2019calibration,marx2024calibration}:
    \begin{equation}
    \label{eq: calibration as divergence}
        \mathrm{CE}_{\mathcal{H}}(f) = \sup_{\phi \in \mathcal{H}}\Big|\mathbb{E} \Big[\phi(f(X), Y)\Big] - \mathbb{E}\Big[\phi(f(X), Z)\Big]\Big|,
    \end{equation}
    where $\mathcal{H}$ is typically chosen to be the unit ball of a reproducing kernel Hilbert space (RKHS) $\mathcal{H}$, and $Z$ is a conditioning variable that is assumed to be calibrated.\\



For a given dataset $\mathcal{D}=\{(x_i, y_i)\}_{i=1}^N$, we denote $\widehat{\mathrm{CE}}(f):= \widehat{\mathrm{CE}}(f, \mathcal{D})$ as an estimator of $\mathrm{CE}$
 based on the dataset $\mathcal{D}$. 

A very common way to estimate the calibration error is to use binning \cite{zadrozny2001,naeini2015,guo2017}, i.e., partitioning the domain of the predictor variable $f(X)$ into a finite number of discrete intervals or \textit{bins}, and estimating the calibration error by aggregating data within each bin. However, the piecewise nature of binning functions makes binned estimators non-differentiable and unsuitable for optimization tasks requiring gradient information, a property that we will need in Section~3. Binning also introduces a bias in estimating the conditional expectation  $\mathbb{E}[Y | f(X)]$,  because it replaces continuous variations with average values within bins. Furthermore, binned estimators usually measure a weaker form of calibration than distribution calibration as defined in Definition \ref{def: strong calibration}, namely classwise calibration, which only demands calibrated marginal probabilities. For these reasons, binning is not discussed in the following overview of existing estimators. Due to the nature of our problem, we make use of calibration estimators that are consistent, (asymptotically) unbiased and differentiable.  \\

\subsection{Overview of calibration errors and estimators}
 One of the first and still widely used metrics \cite{guptacalibration,murphy1973brierdecomposition,rahimi2020intra} that intrinsically also assesses calibration is the (expected) \textit{Brier score} \cite{brier1950}. It is defined as \begin{equation}
    \text{BS}(f) = \mathbb{E}\Big[\Big\|f(X)-\vec{e}_Y \Big\|_2^2 \Big],
\end{equation}
with $\vec{e}_Y$ being the one-hot-encoded vector of $Y$. Its estimator is given by the mean squared error between the predicted probability and actual outcome, averaged over all classes: \begin{equation}
    \widehat{\mathrm{BS}}(f) = \sum_{i=1}^N \sum_{j=1}^K (f_{ij} - \mathbb{I}_{(y_i =j)})^2,
\end{equation}
where $f_{ij}$ denotes the $j$-th entry of the vector $f(x_i)$. Being a proper scoring rule \cite{gneiting2007strictly}, it does not only measure calibration, but can be decomposed into a calibration and a refinement loss term \cite{murphy1973brierdecomposition,kull2015noveldecompositions}, where the calibration loss consists of the expected $L_2$ error between the conditional distribution $P_{X|f(X)}$ and $f(X)$, hence Eq.\ (\ref{eq: calibration as expectation}) with $d$ being the Euclidean distance. In their framework of \textit{proper calibration errors}, \citet{gruber2022betteruncertaintycalibration} show that the Brier score also serves as an upper bound for other common calibration errors.\\

\citet{popordanoska2022consistent} proposed an estimator for the \textit{$L^p$ calibration error}, which for the case $p=2$ directly corresponds to the calibration term in the decomposition of the Brier score: 
\begin{equation}
\label{eq: lp calibration error}
    \mathrm{CE}_p(f) = \Big(\mathbb{E}\Big[\Big\| f(X) - \mathbb{P}_{Y|f(X)}\Big\|_p^p\Big]\Big)^{\frac{1}{p}}.
\end{equation}
They formulate an estimator of $\mathrm{C}\mathrm{E}_{p}(f)$ using Beta and Dirichlet kernel density 
estimates for the binary and multi-class classification case, respectively. Precisely, they prove that 
\begin{equation}
    \widehat{\mathrm{CE}}_p(f)
    =\Big(\frac{1}{n}\sum_{j=1}^{n}\left[\Big\| \widehat{{\mathbb{E}}[y|f(x_j)}] -f(x_{j})\Big\|_{p}^{p}\right]\Big)^{\frac{1}{p}}
\end{equation}
is a point-wise consistent and asymptotically unbiased estimator, where the estimator of the conditional expectation is defined via kernel density estimation.\\


Using the Kullback-Leibler divergence $D_{KL}$ as the divergence measure, 
one can define  
 another proper calibration error \cite{pmlr-v238-popordanoska24a}: \begin{equation}
 \label{eq: kl calibration error}
     \text{CE}_{KL}(f):= \mathbb{E}\Big[D_{KL}(\mathbb{E}[Y|f(X)], f(X))\Big],
 \end{equation}
which exactly forms the calibration error term of the popular log loss. An estimator is given by 
\begin{equation}
    \widehat{\mathrm{CE}}_{KL}(f)
    =\frac{1}{n}\sum_{j=1}^{n}\left\langle \widehat{{\mathbb{E}}[y|f(x_j)}], \log \frac{\mathbb{E}[y|f(x_j)]}{f(x_j)}\right\rangle,
\end{equation}
where the estimator $\widehat{{\mathbb{E}}[y|f(x_j)}]$ is again defined via kernel density estimation.\\


\citet{widmann2019calibration} introduced the \textit{kernel calibration error} and multiple differentiable estimators for it. In its kernel-based formulation, it is defined using a matrix-valued kernel $k: \Delta_K \times \Delta_K \rightarrow \mathbb{R}^{K \times K}$ as follows:
    \begin{eqnarray}
    \label{eq: kernel calibration error}
        \mathrm{CE}_k(f)=\left(\mathbb{E}\left[(\vec{e}_{Y}-f(X))^{\mathsf{T}}k(f(X),f(X^{\prime}))(\vec{e}_{Y^{\prime}}
        -f(X^{\prime}))\right]\right)^{1/2},
    \end{eqnarray}
where $(X', Y')$ is an independent copy of $(X,Y)$ and $\vec{e}_Y$, $\vec{e}_{Y'}$ are the one-hot-encoded vectors of the random variables $Y$ and $Y'$. We will make use of the computationally more feasible unbiased estimator they propose, defined as \begin{equation*}
\widehat{\mathrm{CE}}_k(f) = \frac{1}{\lfloor n/2\rfloor} \sum_{i=1}^{\lfloor n/2 \rfloor} h_{2i-1,2i},
\end{equation*} with $k$ being a universal matrix-valued kernel
and $h_{i,j}$ corresponds to the term in the expectation (\ref{eq: kernel calibration error}) evaluated on two instance-label pairs.\\


\citet{marx2024calibration} introduced a framework of calibration estimators for different types of calibration, which considers it as a distribution matching problem between (true) conditional distribution $\mathbb{P}_{Y|X}$ and the predicted distribution $\hat{\mathbb{P}}_{Y|X}$ induced by $f(X)$. Similar to Widmann et al.\ \cite{widmann2019calibration}, they proposed  integral probability metrics to measure the distance between real conditional and predicted distribution. For $(X, Y) \sim \mathbb{P}$ and $(X, \hat{Y}) \sim \hat{\mathbb{P}}$ they define the calibration error as the Maximum-Mean-Discrepancy (MMD) between $\mathbb{P}$ and $\hat{\mathbb{P}}$:
\begin{equation}
\label{eq: max mean cal error}
    \text{CE}_{MMD}(f) = \sup_{\phi \in \mathcal{H}} \Big|\mathbb{E}\big[\phi(Y, f(X))\big]- \mathbb{E}\big[\phi(\hat{Y}, f(X))\big] \Big|,
\end{equation}
where $\mathcal{H}$ is again the unit ball of an RKHS.
 For the classification case, they introduce a trainable calibration estimate for distribution calibration which measures the squared error as
\begin{equation}
    \widehat{\text{CE}}_{MMD}^2 = \frac{1}{n(n-1)} \sum_{i=1}^n \sum_{j=1, j \neq i}^n h_{ij}, 
\end{equation} where the terms $h_{ij}$ are defined
 dependent on the type of calibration. They define $z_i$ being a conditional random variable distributed according to $f(x_i)$, for $i \in \{1, \dots, N \}$ and
 \begin{eqnarray*}
    h_{ij} = k((y_i, z_i), (y_j, z_j)) + \sum_{y \in \mathcal{Y}}\sum_{y' \in \mathcal{Y}} q_i(y)q_j(y')k((y, z_i), (y', z_j))
    - 2 \sum_{y\in \mathcal{Y}}q_i(y)k((y, z_i), (y_j, z_j)),
 \end{eqnarray*}
 with $k:  \Delta_K \times \Delta_K \rightarrow \mathbb{R}$ being a universal kernel function, and $q_i(y)$ the predicted probability for $y$ given $f(x_i)$.\\


\subsection{Statistical tests for calibration}

As the estimators $\widehat{\mathrm{CE}}(f,\mathcal D)$ reviewed in the previous
section are computed from finite samples, their realised values inherit
sampling variability, and hence comparing them in terms of their realized values does not take this randomness into account. Appendix~\ref{sec: estimator analysis} visualises the
distribution of the estimators under the null hypothesis that they are calibrated; the dispersion
shrinks with the sample size $N$, but even for large $N$ randomness is
non-negligible.  Consequently a formal hypothesis test is needed when
evaluating the calibration of a classifier
$f\!:\mathcal X\!\rightarrow\!\Delta_{K}$ on a finite data set
$\mathcal D=\{(x_i,y_i)\}_{i=1}^{N}$:
\begin{equation*}
H_{0}: \text{$f$ is calibrated} \quad\quad
H_{1}: \neg H_0.
\end{equation*}

There are a number of established tests for calibration in classification problems.
The classical Hosmer-Lemeshow test \cite{hosmer1997comparison}, initially developed as a goodness-of-fit test for the logistic regression model, considers a chi-squared distributed test statistic based on observed and expected frequencies of events. \added{However, as it requires binning and can only be used to test for confidence calibration, we do not include it in this analysis.}
\citet{widmann2019calibration} derive a test for the kernel calibration error
($\text{CE}_k$), which takes the asymptotic normality of their proposed estimator into account.  While it yields analytic $p$-values and is computationally light once the test statistic is calculated, its validity holds only for the specific kernel calibration estimator and \added{its asymptotic normality}. \\

\citet{vaicenavicius19a} developed a general framework to evaluate calibration based on \textit{consistency sampling} \cite{broeckersmith2007}, a method that applies bootstrapping to estimate the distribution of the calibration estimator's values under the null hypothesis that $f$ is calibrated. \added{This is done by resampling new labels in each bootstrap iteration, based on the distribution that $f$ predicts for the respective bootstrap sample. Given the (empirical) distribution function of the calibration estimator under the null hypothesis, one can then check how likely the given calibration estimate is under the assumption that $f$ is calibrated.
Being a nonparametric test, it can therefore be used to test for distribution calibration in combination with any of the calibration estimators mentioned above. For these reasons we adopt the nonparametric bootstrap test of Vaicenavius et al.
\cite{vaicenavicius19a} throughout the paper: it is the only approach
that (a) is valid for all considered calibration errors,
(b) remains distribution-free, and (c) integrates seamlessly with our
proposed algorithm.}

\section{Calibration of sets of probabilistic predictors}
\label{sec: calibration for ensemble models}
We will now come to the main part of this paper, namely the evaluation of
calibration of not only one, but a set of classifier models. To this end, after introducing the necessary methodology, we introduce our proposed test in a step-wise manner.

\subsection{Credal sets} 
\added{\textit{Credal sets}, generally defined as sets of probability distributions, form a way to represent disbelief or uncertainty about the \textit{true} underlying probability distribution. Instead of committing to a single point prediction, using credal sets to represent uncertainty allows to stay imprecise, thereby avoiding prediction error if the uncertainty is too high.}
The sets can be obtained in various ways \cite{wang2025creinns,caprio2024credal}, one direct way being via sets of probabilistic classifiers. In the following, let $\mathcal{F} := \{f^{(1)}, \dots, f^{(M)}\}$, 
with $f^{(i)} : \mathcal{X} \rightarrow \Delta_K$ the $i$-th probabilistic model in a set that contains $M$ models in total.
For each feature vector $x \in \mathcal{X}$, this yields a (credal) set of probability distributions
$\mathcal{F} |_x = \{f^{(1)}(x), \dots, f^{(M)}(x)\} \subseteq \Delta_K$.
A natural way to validate a representation of epistemic uncertainty through sets of predictors is to look at their possible \textit{convex combinations}: If there is at least one prediction in the convex hull of $\mathcal{F}|_x$ that is calibrated, then one can argue that the set of predictors contains the ground truth aleatoric part of the uncertainty, hence fulfils a necessary requirement for representing epistemic uncertainty. Calibration here serves as a relevant necessary condition for validity; a calibrated combination indicates that the predictors are not systematically biased and that they can approximate the true data-generating process in a consistent manner. For each instance $x\in \mathcal{X}$, we now define the (credal) set $\mathcal{S}(\mathcal{F},x)$ as the set of all possible convex combinations of predictors in $\mathcal{F}|_x$.
\begin{definition}
\label{def:credalset}
For a feature vector $x \in \mathcal{X}$, the credal set $\mathcal{S}(\mathcal{F}, x)$ is the set of all convex combinations of $\mathcal{F} |_x$:
\begin{equation*}
\label{eq: credal set}
\mathcal{S}(\mathcal{F}, x) = \Big \{f_{\boldsymbol{\lambda}}(x) \in \mathcal{H} \Big| \, f_{\boldsymbol{\lambda}}(x) = 
    \sum_{i=1}^M \lambda_i(x)f^{(i)}(x) \, \Big| \,(\lambda_1, \dots, \lambda_M) \in \Delta_{M, \mathcal{X}} \Big\}, 
\end{equation*}
where 
    $\Delta_{M, \mathcal{X}} = \Big\{\boldsymbol{\lambda} = (\lambda_1, \dots, \lambda_M) \Big|\, \lambda_i: \mathcal{X} \mapsto [0, 1] \, \text{and}\, \sum_{i=1}^M \lambda_i(x) = 1 \, \forall x \in \mathcal{X} \Big\}$
denotes the set of all functions with co-domain being the $(M-1)$-simplex.
\end{definition}
Here $\Delta_{M,\mathcal X}$ is a function space over $\mathcal X$, while in previous work \cite{mortier2023calibration}, $\Delta_{M}$ was the $(M-1)$-simplex of all (constant) convex combinations, that is, the same convex combination for all $x \in \mathcal{X}$ was analysed in order to test for calibration.
Instance-dependent convex combinations offer the flexibility needed to capture varying degrees of aleatoric and epistemic uncertainty across the input domain. Different regions of the instance space may require different degrees of model blending to appropriately represent the level of uncertainty, particularly when some models are better calibrated than others. Def.~\ref{def:credalset} constructs credal sets via realizations of sets of predictors, which yield sets of probability distributions. From the perspective of imprecise probabilities, one might also define a \textit{credal} predictor, which yields as output a set of probability distributions. The proposed framework also holds for this scenario, given the predictor outputs a \textit{convex set} of extreme points.


\subsection{Validity and calibration for sets of predictions}
\label{sec: validity for credal sets}
In the following, we formally define our notion of validity for set representations of epistemic uncertainty. It implies that for each $x \in \mathcal{X}$, the true underlying probability distribution $\prob_{Y|X=x}$ is contained in the (credal) set of all possible convex combinations $\mathcal{S}(\mathcal{F}, x)$. 

\begin{definition}[Validity of credal sets]
    Let $\mathcal{F}= \{f^{(1)}, \dots, f^{(M)}\}$ be a set of predictors,  with $f^{(i)}: \mathcal{X} \rightarrow \mathbb{P}(\mathcal{Y})$, and for each $x\in \mathcal{X}$, define the induced credal set $\mathcal{F}|_x$ as in Definition \ref{def:credalset}. Then $\mathcal{F}$ is \textit{valid} if
        $\prob_{Y|X=x} \in \mathcal{S}(\mathcal{F}, x)$    for all $x \in \mathcal{X}$.
\end{definition}

In practice, testing validity requires having either access to the ground truth conditional distribution $\mathbb{P}_{Y|X}$ or a sufficient number of samples thereof, as assumed in \cite{chau2024credal}.
Hence, in our setting, we make use of the \textit{weaker} notion of calibration which results from conditioning on the predictions $f(X)$ instead of $X$. This way, we analyse whether there is a $\boldsymbol{\lambda} \in \Delta_{M, \mathcal{X}}$ such that $\mathbb{P}_{Y|f_{\lambda}(x)} \in \mathcal{S}(\mathcal{F}, x)$. Similar as in \cite{mortier2023calibration}, we now define a set of classifiers, or equivalently a credal classifier, as \textit{calibrated} if there \textit{exists} (at least) one calibrated convex combination of predictors. Definition \ref{def: calibration of credal sets} forms a generalization of Definition 4 of \cite{mortier2023calibration}, where the coefficients of the convex combination do not form functions on the instance space, but are \textit{constants}, i.e., $\boldsymbol{\lambda} \in \Delta_M$.

\begin{definition}
\label{def: calibration of credal sets}
    Let $\mathcal{F}=\{f^{(1)}, \dots, f^{(M)}\}$ be a set of predictors, $f^{(i)}: \mathcal{X} \rightarrow \Delta_K$. We say that $\mathcal{F}$ is \textit{calibrated in distribution} if there exists $\boldsymbol{\lambda} \in \Delta_{M, \mathcal{X}}$ such that the \textit{combined predictor} $f_{\boldsymbol{\lambda}} : \mathcal{X} \rightarrow \mathcal{S}(\mathcal{F}, \mathcal{X})$ defined as
    $$f_{\boldsymbol{\lambda}}(x) = \sum_{i=1}^M \lambda_i (x) f^{(i)}(x)$$
is calibrated in distribution (Definition \ref{def: strong calibration}).
\end{definition}

Note that calibration provides a necessary, but not \textit{sufficient} condition for validity: A valid predictor is always calibrated, yet the reverse does not have to hold. In fact, there are (possibly) many calibrated predictors, as shown in \cite{vaicenavicius19a}. In Appendix \ref{sec: many calibrated convex combinations}, we show that this is also the case for convex combinations of probabilistic models. However, if a predictor is \textit{not} calibrated, one knows that it is also not valid.
\subsection{A novel calibration test for sets of predictors}
Our null and alternative hypothesis for testing the calibration of a set of predictors $\mathcal{F}$ can now be formulated as 
\begin{eqnarray}
\label{eq: H0 and H1 credal sets}
    H_0: \quad \exists \boldsymbol{\lambda} \in \Delta_{M, \mathcal{X}} \, \text{s.t.} f_{\boldsymbol{\lambda}} \, \text{is calibrated}, \quad H_1: \neg H_0.
\end{eqnarray}
In order to analyse whether there is a calibrated convex combination, a natural approach is to analyse the one with the lowest calibration error. Hence, 
we define the \textit{minimum calibration error} as follows:
\begin{equation}
\label{eq: minimization lambda}
    \min_{\boldsymbol{\lambda} \in \Delta_{M, \mathcal{X}}}  g(f_{\boldsymbol{\lambda}}),
\end{equation}
where $g$ is a calibration error.\\

For the experiments, we choose $g$ to be equal to the errors proposed in Section~2, i.e., $g \in \{\text{CE}_p, \text{CE}_{KL}, \text{CE}_{MMD},\text{CE}_k\}$. These have under suitable conditions the desirable property that $f$ is calibrated if and only if $\text{CE} = 0$, making them suitable for optimization. Furthermore, the respective estimators described in Section \ref{sec: calibration errors} are both consistent and at least asymptotically unbiased. Having an optimisation dataset $\mathcal{D}_{opt} = \{(x_i, y_i)\}_{i=1}^{\tilde{N}}$, 
one can represent the evaluations of the weight function $\boldsymbol{\lambda}: \mathcal{X}\rightarrow \Delta_{M, \mathcal{X}}$ by an $\tilde{N} \times M$ matrix:
\begin{equation}
   \Lambda =  \begin{pmatrix}
    \lambda_1(x_1) & \dots & \lambda_M(x_1) \\
    \vdots & \ddots & \vdots \\
    \lambda_1(x_{\tilde{N}}) & \dots & \lambda_M(x_{\tilde{N}}) \\

    \end{pmatrix}.
\end{equation}
Hence, for $\hat{g}$ being the (data-dependent) estimator of $g$, finding the \textit{minimum empirical calibration error} can be formulated as follows: 
\begin{equation}
\label{eq: minimum calibration error}
    \min_{\Lambda_{i, j}, \, i \in {1, \dots, M}, j \in {1, \dots, \tilde{N}}}\hat{g}(f_{\Lambda}, \mathcal{D}_{val})
\end{equation}
with $f_{\Lambda}(x_i)=\sum_{m=1}^M \lambda_m(x_i) f^{(m)}(x_i) \in [0,1]^K$, having the constraint $\sum_{j=1}^M \Lambda_{ij} =1$ for all  $j \in \{1, \dots, \tilde{N}\}$.\\

As already described in Section \ref{sec: validity for credal sets}, there are in general many calibrated functions $f$, implying the absence of a unique global minimum for the optimization problem (\ref{eq: combined calibration loss}). Therefore, in the optimisation, we make use of a combined loss function consisting of a proper scoring rule and the respective calibration estimator, weighted by a constant. Proper scoring rules intrinsically optimise for both a calibration error and an accuracy term and provide a stable optimisation. The specific optimisation method is described in Section \ref{sec: finding the most calibrated convex combination}.

\subsection{Algorithmic procedure}
\begin{algorithm}[t]
\caption{Algorithm to test whether there exists a calibrated, instance-dependent convex combination of an ensemble of $M$ classifier models.}
\label{alg: Alg v2.0}
\begin{algorithmic}[1]
    \State \textbf{Input:} validation set including features and labels $\mathcal{D}_{val}= \{(x_i, y_i)\}_{i=1}^N$, instance-wise evaluated point predictors  $\mathcal{F} |_{(x_i)_{i=1}^N} = \big(f^{(1)}(x_i), \dots, f^{(M)}(x_i)\big)_{i=1}^N$, estimator of calibration error $\hat{g}$, confidence level $\alpha$, number of bootstrapping iterations $D \in \mathbb{N}$
    \State \textbf{Output:} boolean stating whether to reject $H_0$, with $H_0$ as in (\ref{eq: H0 and H1 credal sets})
    \State $\Lambda^* \gets \displaystyle \argmin_{\Lambda_{i, j}, \, i \in \{1, \dots, M\}, j \in \{1, \dots, N\}}\hat{g}(f_{\Lambda}, \mathcal{D}_{val})$
    \State $f_{\Lambda^*} \gets \Lambda^* \cdot \mathcal{F}$
    
    \For{$d = 1, \dots, D$}
        \State $\mathcal{D}_d \gets \{\tilde{x}_1, \dots, \tilde{x}_N\}$ bootstrap sample of instances
        \For{$n = 1, \dots, N$}
            \State sample $\tilde{y}_n \in \{1, \dots, K\} \sim \text{Cat}(f_{\Lambda^*}(x_n))$
        \EndFor
        \State $\tilde{\mathcal{D}}_{val} \gets \{(x_i, \tilde{y}_i)\}_{i=1}^N$ \Comment{replace labels of validation set}
        \State $t_{0, d} \gets \hat{g}(f_{\Lambda^*}, \tilde{\mathcal{D}}_{val})$
    \EndFor
    
    \State $q_{1-\alpha} \gets (1-\alpha)$-quantile of empirical CDF $\hat{F}$ of $(t_{0,1}, \dots, t_{0, D})$
    \State $t \gets \hat{g}(f_{\Lambda^*}, \mathcal{D}_{val})$
    \If{$t > q_{1-\alpha}$}
        \State reject $H_0$
    \Else
        \State do not reject $H_0$
    \EndIf
\end{algorithmic}
\end{algorithm}
We now explain in detail our new, adapted version of the statistical test proposed in Mortier at al.\ \cite{mortier2023calibration}, which in turn is based on the bootstrap test introduced by Vaicenavicius et al. \cite{vaicenavicius19a}. 
Algorithm \ref{alg: Alg v2.0} shows the pseudo code of the test, which consists of the following steps:\\

First, using an appropriate optimization algorithm, the per-instance convex combination $\Lambda^*$ with minimal calibration error is found (line 3).
     Using $\Lambda^*$, the predictions of the combined classifier $f_{\Lambda^*}$ are computed (line 4).
    Third, the calibration test of Vaicenavicius et al.\ \cite{vaicenavicius19a} is employed to assess whether the predictions of $f_{\Lambda^*}$ are calibrated (line 5-19). This statistical test is based on bootstrapping  and \textit{consistency resampling}, that is, resampling the labels from the distribution induced by $f_{\Lambda^*}$. This way, one is able to estimate the distribution of the calibration error estimator $\hat{g}$ under the null hypothesis that $f_{\Lambda^*}$ is calibrated (see also Figure \ref{fig: histogram estimators h0} for an example).

 For a given significance level $\alpha$, the test rejects the null hypothesis if the value of the calibration error estimator is higher than the $(1-\alpha)$-quantile of the bootstrapped empirical distribution. In the other case, it cannot reject it.
Algorithm \ref{alg: Alg v2.0} differs from the proposed algorithm in Mortier at al.\ \cite{mortier2023calibration} in two important aspects:
\begin{enumerate}
    \item It allows for the case where the weights for the most calibrated convex combination \textit{depend on the instance space}, i.e., for different regions in it, it accounts for the fact that some predictors might be better calibrated than others and vice versa.
    \item It uses an \textit{optimization} instead of a sampling-based approach: For constructing the distribution under the null hypothesis, the previous test used uniform sampling of the weights. Directly optimizing over the calibration error and performing the test on the found convex combination has two main
    advantages: 
    (a), the uniform sampling of weights used in \cite{mortier2023calibration} does not lead to uniform sampling in the credal set, which we further elaborate on in Appendix \ref{sec: uniform sampling}. By putting emphasis on combinations in the inner of the polytope, it could lead to a biased estimate of the distribution under the null hypothesis. Advanced sampling techniques that could result in a more uniform sampling in the credal sets, like rejection-sampling or triangulizations, are computationally costly. (b), the minimization step automatically "guides" the algorithm towards the region with low calibration errors, avoiding the need to explore the whole credal set.
\added{In Appendix \ref{sec: validity} we show that using a universal approximator to learn the underlying convex combination leads to the test being asymptotically valid, i.e., controlling the Type $1$ error. The optimisation step is explained in more detail in the following section.}
\end{enumerate}

\subsection{A robust way of finding the most calibrated convex combination}
\label{sec: finding the most calibrated convex combination}
\begin{figure}[t]
    \centering
    \includegraphics[width=.9\textwidth]{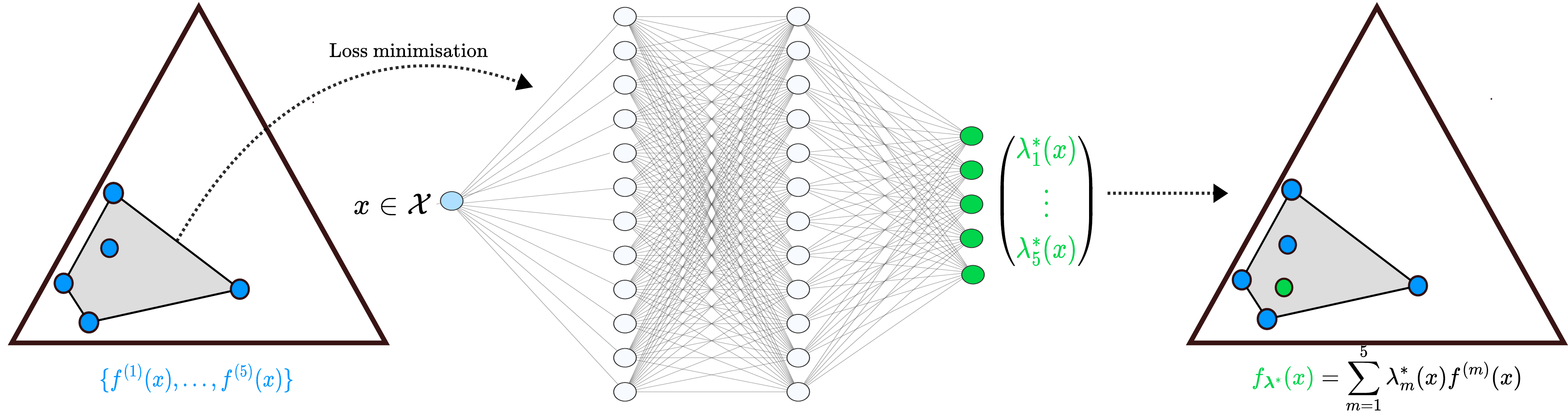}
    \caption{\added{General concept of learning the optimal function $\boldsymbol{\lambda}^*: \mathcal{X} \rightarrow \Delta_{M, \mathcal{X}}$, here for the example of $M=5$ ensemble predictions.} The neural network is trained using the combined calibration loss function as in Eq.\ (\ref{eq: combined calibration loss}), with the calibration estimators introduced in Section \ref{sec: calibration errors} for the calibration term. It predicts for a given instance value $x\in \mathcal{X}$ the optimal $\boldsymbol{\lambda}^*(x)$ such that the empirical calibration error for the combined predictor $f_{\boldsymbol{\lambda}^*}$ is minimized.}
    \label{fig: mlp optimization}
\end{figure}
We will now analyse further how optimization problem (\ref{eq: minimum calibration error}) can be solved in an efficient and robust manner. Specifically, we try to avoid the problem of \textit{overfitting} on the empirical calibration error on the respective dataset: Since the latter is finite, there might be a mismatch between the estimated and the population-level calibration error. Classical gradient-based solvers, which purely optimise the calibration error on the same dataset that is used for the test, might therefore run into the problem of underestimating the true calibration error, thereby making the test more conservative. This can also be seen in the empirical results of \citet{mortier2023calibration}. 
 Therefore, we suggest an alternative approach find the minimum in (\ref{eq: minimum calibration error}), which incorporates two important aspects:
 \begin{itemize}
     \item We use a neural network for learning the weight function $\boldsymbol{\lambda}: \mathcal{X} \rightarrow \Delta_{M, \mathcal{X}}$, exemplary visualized in Figure \ref{fig: mlp optimization}. This approach is similar to stacking \cite{wolpert1992stacked}, where a second meta-learner is trained to find the optimal combination of predictions. It is trained on a separate optimization dataset, and then used to predict the optimal convex combination for our test.
     \item As a loss function, we use the risk of a proper scoring rule $\ell$, combined with a term controlling the calibration error, such that the true risk minimizer is given by \begin{equation}
  \label{eq: combined calibration loss}
    \boldsymbol{\lambda}^* \in \argmin_{\boldsymbol{\lambda}\in \Delta_{M, \mathcal{X}}} \mathbb{E}\big[\ell(f_{\boldsymbol{\lambda}}(x), y)\big] + \gamma \cdot g(f_{\boldsymbol{\lambda}}).
\end{equation}
We use $\hat{g} \in \{\widehat{\mathrm{CE}}_{KL}, \widehat{\text{CE}}_2, \widehat{\text{CE}}_k, \widehat{\text{CE}}_{MMD}\}$ as estimates for the respective calibration error of interest, while $\gamma \geq 0$ denotes a weight factor. This approach takes recent insights \cite{popordanoska2022consistent,marx2024calibration} that adding a calibrating penalty to a proper scoring rule can help ensuring calibration, into account, while avoiding trivial or degenerate solutions. In particular, many convex combinations can be calibrated (Appendix \ref{sec: many calibrated convex combinations}), hence, the proper scoring rule serves as a \textit{regularizer}, guiding the optimizer toward solutions that not only reduce calibration error but also maintain high predictive accuracy. Furthermore, we learn $\boldsymbol{\lambda}^*$ on a validation set but evaluate calibration on a separate test set. This split reduces the risk of overfitting to our chosen calibration metric, ensuring the resulting test does not become overly conservative. 
 \end{itemize}
 
\section{Experimental results}
\label{sec: experimental results}
We evaluate our test both on synthetic and real-world data. A more detailed description of the experimental setup can also be found in Appendix \ref{sec: experimental setup} and the code for reproducing the experiments is available on GitHub.\footnote{\url{https://github.com/mkjuergens/EnsembleCalibration}}
For the experiments on synthetic data, a MLP architecture with  $3$ hidden layers and $16$ neurons is used. \added{Experimentally, we found that this simple architecture was sufficient to solve the optimization problem in Eq. (\ref{eq: combined calibration loss}) and learn the underlying weight function. For the experiments on real data, we use a more complex network architecture, which is described in Section \ref{sec: experiments real data}.} As a loss function, we use the combined loss as in (\ref{eq: combined calibration loss}), with the Brier score as a proper scoring rule, and $\gamma=0.01$.

\subsection{Binary classification}
\label{sec: binary classification}
For illustration purposes, we first examine the case of binary classification by simulating $M=2$ probabilistic predictors, each outputting a probability for the positive class. For each input $x$, the predictor's probabilities $f^{(1)}(x)$ and $f^{(2)}(x)$ are sampled from a Gaussian process (GP) constrained to $[0,1]$. We then define a \textit{true} calibrated predictor $f^*$ that lies inside (under $H_0$) or outside (under $H_1$) the convex hull of the two predictors.
Figure \ref{fig: gp calibration cases} shows the exemplary setting for this experiment. Under $H_0$, we consider two cases of non-instance and instance-dependence of $\boldsymbol{\lambda}=(\lambda_1, \lambda_2)^T$:
\begin{enumerate}
    \item $\boldsymbol{H_{0,1}}$: The calibrated predictor $f^*$ is a \textit{constant} convex combination of $f^{(1)}$ and $f^{(2)}$,
    where $\lambda_1^*(x) \equiv c \sim Unif([0,1])$, $\lambda_2^*(x) = 1 - \lambda_1^*(x)$.
    \item $\boldsymbol{H_{0,2}}$: $\lambda^*(x)$ is a randomly generated polynomial function and $f^*(x)= \lambda^*(x) f^{(1)}(x) +(1-\lambda^*(x))f^{(2)}(x)$.
\end{enumerate}
Under $H_1$, $f^*(x)$ lies strictly outside or near the boundary of the credal set. We generate $3$ scenarios of increasing distance from it: 
\begin{enumerate}
    \item $\boldsymbol{H_{1,1}}$: $f^*(x)$ is set at a small $\epsilon$-distance to one boundary.
    
    \item $\boldsymbol{H_{1,2}}$: is sampled from a GP that remains close but outside the credal set.
    \item $\boldsymbol{H_{1,3}}$: similar to $\boldsymbol{H_{1,2}}$, but allowing a larger maximum distance.
\end{enumerate}
\begin{figure}[t]
    \centering
    \begin{subfigure}{.49\textwidth}
        \centering
        \includegraphics[width=\textwidth]{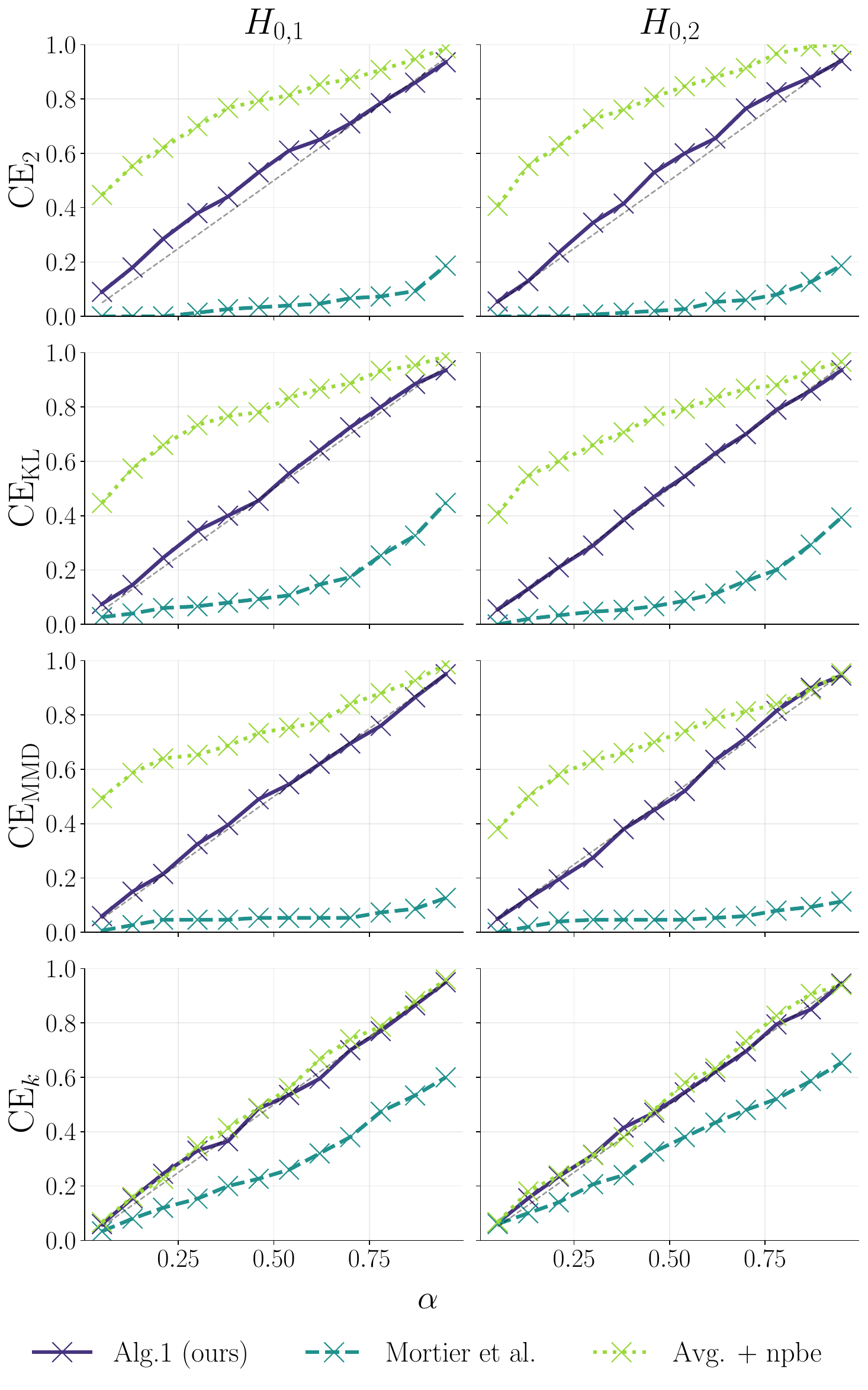}
     \caption{average Type 1 error}
    \end{subfigure}
    \begin{subfigure}{.49\textwidth}
        \centering
        \includegraphics[width=\textwidth]{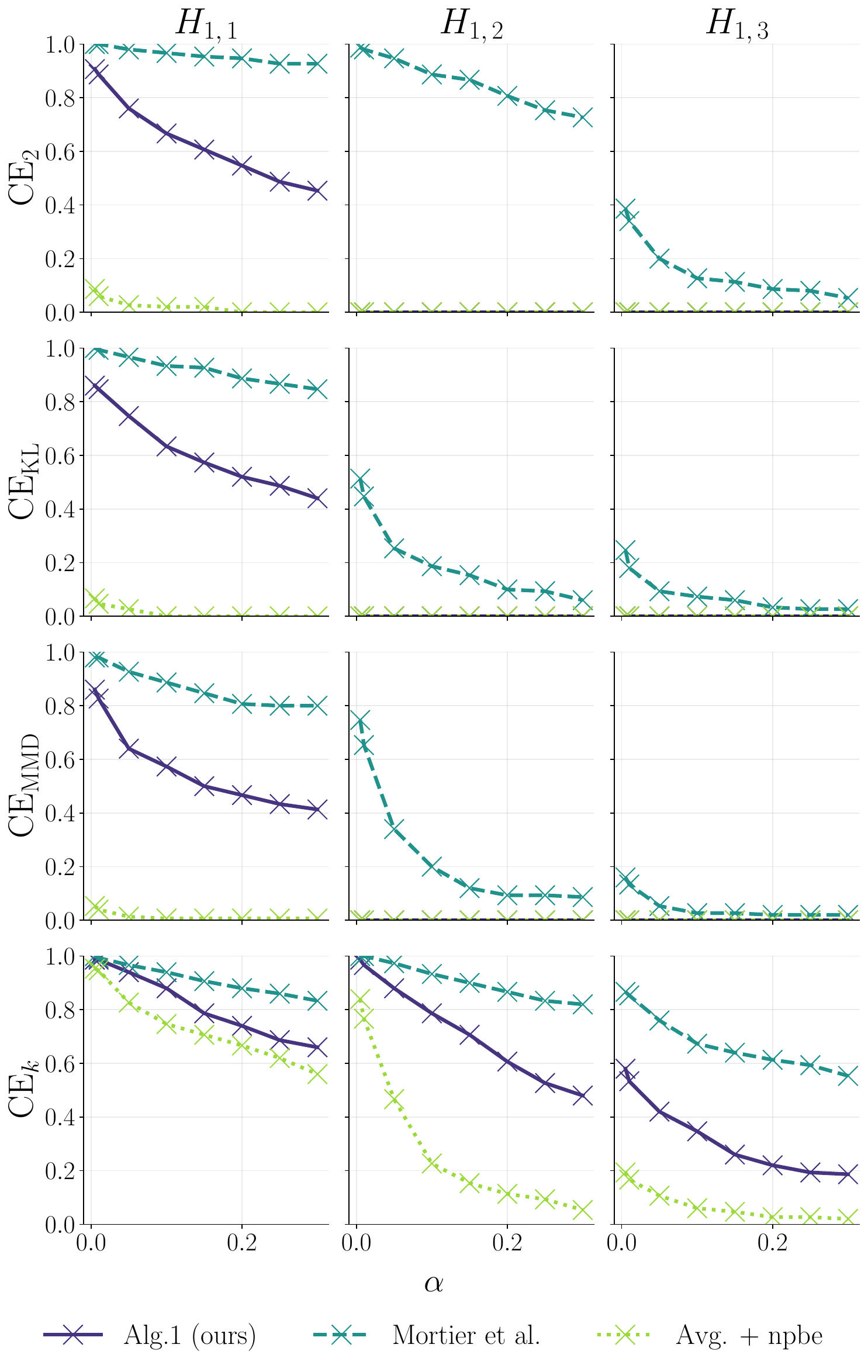}
        \caption{average Type 2 error}
    \end{subfigure}
    \caption{\added{Average Type $1$ and Type $2$ error given a significance level ($x$-axis) for Algorithm \ref{alg: Alg v2.0} with $D=100$ bootstrap iterations run on the binary classification experiment and the cases $\boldsymbol{H}_{0,1} - \boldsymbol{H}_{1,3}$ where the null hypothesis is true (a), and false (b), respectively. The average is taken over $200$ resampling iterations. In each iteration we newly generate a dataset of size $N_{test}=400$ on which the tests are performed. For the proposed test, we do the optimisation on a separate dataset $\mathcal{D}_{opt}$ of the same size. For comparison, the average Type $1$ and Type $2$ error of the previously proposed test (Mortier et al.) as well as the bootstrapping-based test of Vaicenavicius et al. applied to the mean prediction (Avg. + npbe) are shown.}}
    \label{fig: type12 error gp experiment}
\end{figure}
Using the described scenarios, Figure \ref{fig: type12 error gp experiment} shows the Type $1$ and Type $2$ error of the proposed test, averaged over $200$ runs of the experiment. For comparison, we include the results  given by the algorithm of \citet{mortier2023calibration}, as well as the algorithm of \citet{vaicenavicius19a} applied to the mean of the ensemble, $f_{\lambda_{AVG}}(x) = \frac{1}{M}\sum_{i=1}^M f^{(i)}(x)$. The average Type $1$ error  of our testlies close to the chosen significance level. In some cases, it lies slightly above it, which is is due to the generalization error on the unseen data: Since we learn the optimally calibrated convex combination on a separate dataset, the learned convex combination on the test set might be not the one with lowest calibration error, due to the inherent randomness within the data. \added{Nonetheless, with increasing sample size, we expect the Type $1$ error to be controlled by the significance level (due to the test's asymptotic validity). The results show that the test of Mortier et al. is overly conservative, with a rejection rate far lower than the significance level. The test of Vaicenavicius applied to the mean prediction has a Type $1$ error far higher than the significance level; which is expected, as the mean prediction might not lie close enough to the underlying conditional distribution for the test to not reject.
Both our test and the test of Vaicenavicius applied to the mean show a high power for this setting, with almost zero Type $2$ error also in the case where the true distribution lies in very close distance to the polytope. However, the conservative nature of the test of Mortier et al. leads to a higher Type $2$ error, thereby revealing the provided benefit of our test in terms of power.}

\subsection{Multi-class classification}
We consider $K > 2$ classes, and follow a Dirichlet-based scheme to generate an ensemble of $M$ probabilistic predictors. Specifically, we generate a prior $\boldsymbol{p} \sim Dir(\boldsymbol{1})$ and predictions $f^{(1)}(x), \dots, f^{(M)}(x) \sim Dir(\frac{\boldsymbol{p}|x \cdot K}{u(x)})$, where $u: \mathcal{X} \rightarrow \mathbb{R}_{>0}$ is a function defining the epistemic uncertainty (i.e.\ the larger $u$, the higher is the \textit{spread} between the predictions within the set) over the instance space. For the experiments, we use a constant $u(x) \equiv 0.5$. Under $H_0$, which we found to be a good tradeoff to evaluate Type 1 and 2 error of the test. We distinguish between
\begin{enumerate}
    \item  $\boldsymbol{H_{0,1}}$:  $\boldsymbol{\lambda^*}(x) \equiv \boldsymbol{c} \in [0,1]$ where $\boldsymbol{c} \sim Dir(1, \dots,1) \in \Delta_M$ ($\boldsymbol{\lambda^*}$ is constant across the instance space), and
    \item $\boldsymbol{H_{0,2}}$: $\boldsymbol{\lambda^*}(x) = (\lambda_1^*(x), \dots, \lambda_M^*(x)) \in \Delta_M$ with $\lambda^*_m(x) = \sum_{i=0}^D \beta_i x^i$ for $m =1, \dots,M$ and $\sum_{m=1}^M \lambda_m^*(x) =1$ (the components of $\boldsymbol{\lambda^*}$ form scaled polynomials of a certain degree). 
\end{enumerate}
For the alternative hypothesis, we select a random corner $f_c \notin \mathcal{F}$ of the simplex and then set the true underlying conditional distribution $f^*$ as a point on the  connecting line between the corner and the boundary point $f_b \in \mathcal{F}$ of the credal set: $f^*(x) = \delta \cdot f_c(x) + (1- \delta) \cdot f_b(x)$. We define three cases with increasing distance to the credal set by varying the mixing coefficient $\delta$: $\boldsymbol{H_{1,1}}$: $\delta = 0.01$,  $\boldsymbol{H_{1,2}}$: $\delta = 0.1$ and  $\boldsymbol{H_{1,3}}$: $\delta = 0.2$.\\

The resulting Type $1$ and Type $2$ errors of this experiment are shown in Figure \ref{fig: type 1 error dirichlet experiment}, again for our test and the two baselines. We see that the proposed test yields more reliable decisions, not heavily exceeding the given significance level while following it closely. The test of Mortier et al.\ is again overly conservative for all estimators, except for the kernel calibration error estimator $\widehat{\text{CE}}_k$, which in general leads to unreliable results and low power.
\added{While the test of Vaicenavicius et al. applied to the mean prediction yields the highest power, it cannot be seen as a valid test, as it does not control the Type $1$ error rate.
Note here that we use the same datasets for all three tests, but we perform the optimisation for our proposed test on a separate dataset of the same size.
The Type $2$ error analysis shows that, except for the kernel calibration error, our proposed test has higher power than the previously proposed one, and it aligns better with the chosen significance level.}

\begin{figure}[t]
    \centering
    \begin{subfigure}{.49\textwidth}
        \centering
        \includegraphics[width=\textwidth]{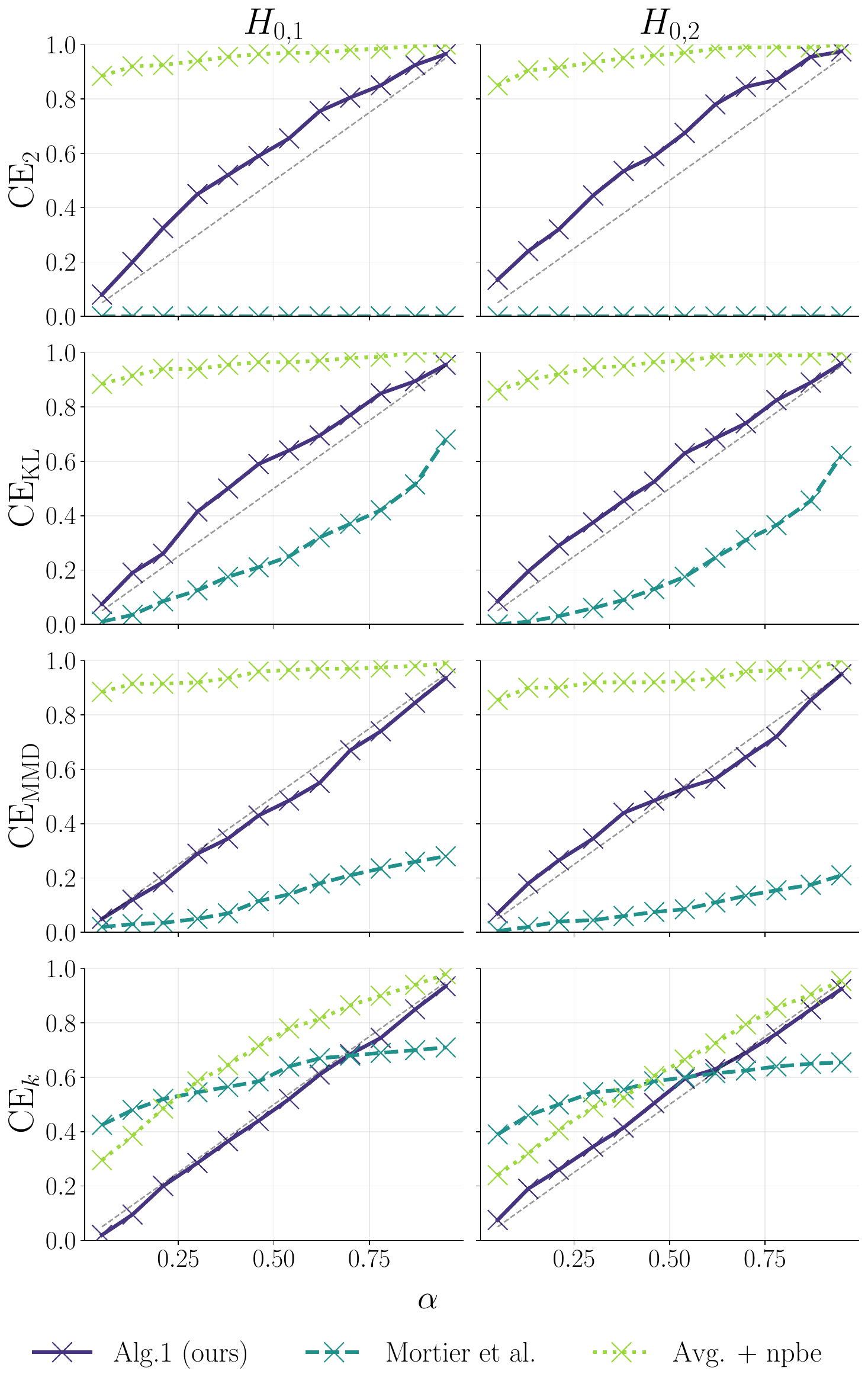}
     \caption{average Type $1$ error}
    \end{subfigure}
    \begin{subfigure}{.49\textwidth}
        \centering
        \includegraphics[width=\textwidth]{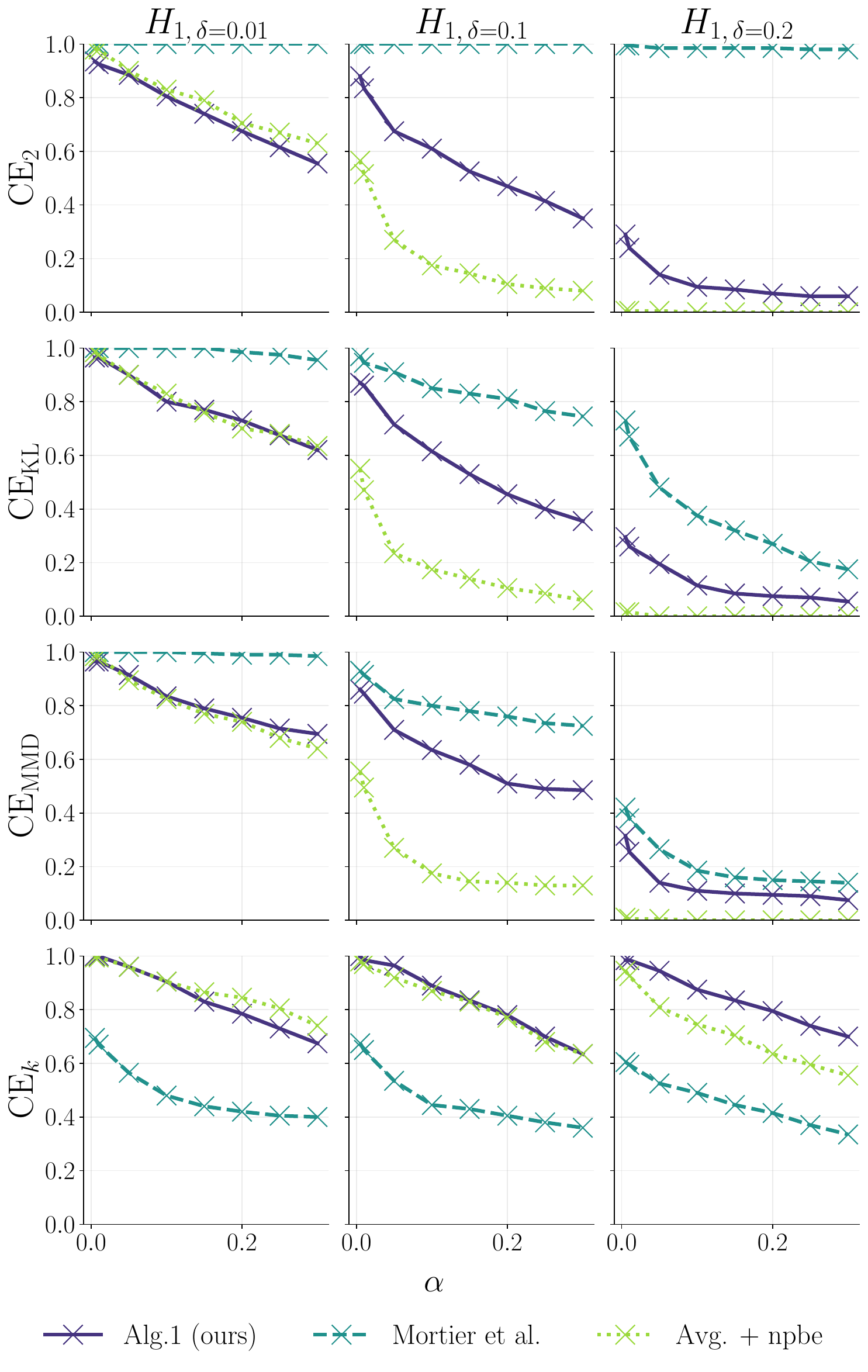}
        \caption{average Type $2$ error}
    \end{subfigure}
    \caption{\added{Average Type $1$ and Type $2$ error given a significance level ($x$-axis) for Algorithm \ref{alg: Alg v2.0} with $D=100$ bootstrap iterations run on the multi-class classification experiment, and the settings $\boldsymbol{H}_{0,1} - \boldsymbol{H}_{1,3}$ where the null hypothesis is true (a), and false (b), resp. with $N_{test}=400$ instances, $M=10$ ensemble members, $K=5$ classes and uncertainty $u=0.5$. The average is taken over $200$ resampling iterations. For comparison, the average Type $1$ and Type $2$ error of the previously proposed test (Mortier et al.) as well as the bootstrapping test applied to the mean prediction (Avg. + npbe) are shown.}}
    \label{fig: type 1 error dirichlet experiment}
\end{figure}


\subsection{Experiments on real-world data}

\begin{table}[t]
\centering
\resizebox{.9\linewidth}{!}{%
\begin{tabular}{c|c|c|c|cccc|cccc}
\toprule
DATA & ARCH. & TYPE & ERR.
      & \multicolumn{4}{c|}{$\boldsymbol{\lambda}_{\mathrm{AVG}}$} 
      & \multicolumn{4}{c}{$\boldsymbol{\lambda}^*$} \\
\cmidrule(lr){5-8} \cmidrule(lr){9-12}
& & & 
& Acc. & $H_0$ & $p$--val. & $\hat{g}$
& Acc. & $H_0$ & $p$--val. & $\hat{g}$ \\
\midrule
\multirow{24}{*}[7ex]{\rotatebox[origin=c]{90}{%
    \parbox{2cm}{\centering CIFAR-$10$}}} & \multirow{12}{*}[4ex]{\rotatebox[origin=c]{90}{\parbox{2cm}{\centering ResNet-$18$}}}
 & \multirow{2}{*}{DE} 
   & $\text{CE}_{KL}$   
     & 0.869 & \(\neg\)REJ. & 1.000 & 0.259  
     & 0.873 & \(\neg\)REJ. & 0.990 & 0.262 \\[0.3em]
& &  & $\text{CE}_2$   
     & 0.869 & \(\neg\)REJ. & 0.970 & 0.125  
     & 0.872 & \(\neg\)REJ. & 0.950 & 0.121 \\[0.3em]
     
\cmidrule(lr){3-12}
& & \multirow{2}{*}{DN($0.2$)} 
   & $\text{CE}_{KL}$
     & 0.847 & \(\neg\)REJ. & 0.810 & 0.293
     & 0.848 & \(\neg\)REJ. & 0.780 & 0.297 \\[0.3em]
& &  & $\text{CE}_2$    
     & 0.847 & \(\neg\)REJ. & 0.800 & 0.129
     & 0.849 & \(\neg\)REJ. & 0.760 & 0.295 \\[0.3em]

\cmidrule(lr){3-12}
& & \multirow{2}{*}{DN($0.5$)} 
   & $\text{CE}_{KL}$   
     & 0.507 & \(\neg\)REJ. & 0.800 & 0.330
     & 0.507 & \(\neg\)REJ. & 0.780 & 0.329 \\[0.3em]
& &  & $\text{CE}_2$  
     & 0.507 & \(\neg\)REJ. & 0.220 & 0.104
     & 0.507 & \(\neg\)REJ. & 0.350 & 0.106 \\[0.3em]
     
\cmidrule(lr){2-12}
& \multirow{12}{*}[4ex]{\rotatebox[origin=c]{90}{\parbox{2cm}{\centering VGG-$19$}}}
 & \multirow{2}{*}{DE} 
   & $\text{CE}_{KL}$  
     & 0.900 & \(\neg\)REJ. & 0.530 & 0.176
     & 0.902 & \(\neg\)REJ. & 0.420 & 0.166 \\[0.3em]
& &  & $\text{CE}_2$     
     & 0.900 & \(\neg\)REJ. & 0.150 & 0.079
     & 0.901 & \(\neg\)REJ. & 0.180 & 0.084 \\[0.3em]
     
\cmidrule(lr){3-12}
& & \multirow{2}{*}{DN($0.2$)} 
   & $\text{CE}_{KL}$ 
     & 0.903 & REJ.         & 0.000 & 0.134
     & 0.903 & REJ.         & 0.000 & 0.133 \\[0.3em]
& &  & $\text{CE}_2$  
     & 0.903 & REJ.         & 0.000 & 0.042
     & 0.903 & REJ.         & 0.000 & 0.042 \\[0.3em]
     
\cmidrule(lr){3-12}
& & \multirow{2}{*}{DN($0.5$)} 
   & $\text{CE}_{KL}$   
     & 0.899 & REJ.         & 0.000 & 0.131
     & 0.896 & REJ.         & 0.000 & 0.129 \\[0.3em]
& &  & $\text{CE}_2$     
     & 0.899 & REJ.         & 0.010 & 0.047
     & 0.896 & REJ.         & 0.000 & 0.056 \\[0.3em]

\midrule
\midrule
\multirow{24}{*}[7ex]{\rotatebox[origin=c]{90}{%
    \parbox{2cm}{\centering CIFAR-$100$}}} & \multirow{12}{*}[4ex]{\rotatebox[origin=c]{90}{\parbox{2cm}{\centering ResNet-$18$}}}
 & \multirow{2}{*}{DE} 
   & $\text{CE}_{KL}$   
   & 0.586 & \(\neg\)REJ. & 1.000 & 1.200
     & 0.589 & \(\neg\)REJ. & 1.000 & 1.203 \\[0.3em]
& &  & $\text{CE}_2$   
     & 0.586 & REJ.         & 0.000 & 0.476
     & 0.589 & REJ.         & 0.000 & 0.482 \\[0.3em]
     
\cmidrule(lr){3-12}
& & \multirow{2}{*}{DN($0.2$)} 
   & $\text{CE}_{KL}$
     & 0.264 & \(\neg\)REJ. & 1.000 & 2.123
     & 0.265 & \(\neg\)REJ. & 1.000 & 2.126 \\[0.3em]
& &  & $\text{CE}_2$    
     & 0.264 & REJ.         & 0.000 & 0.552
     & 0.265 & REJ.         & 0.000 & 0.551 \\[0.3em]
\cmidrule(lr){3-12}
& & \multirow{2}{*}{DN($0.5$)} 
   & $\text{CE}_{KL}$   
        & 0.352 & \(\neg\)REJ. & 1.000 & 2.123
     & 0.353 & \(\neg\)REJ. & 1.000 & 2.136 \\[0.3em]
& &  & $\text{CE}_2$  
      & 0.352 & REJ.         & 0.000 & 0.620
     & 0.351 & REJ.         & 0.030 & 0.624 \\[0.3em]
     
\cmidrule(lr){2-12}
& \multirow{12}{*}[4ex]{\rotatebox[origin=c]{90}{\parbox{2cm}{\centering VGG-$19$}}}
 & \multirow{2}{*}{DE} 
   & $\text{CE}_{KL}$  
     & 0.647 & \(\neg\)REJ. & 0.510 & 0.925
     & 0.648 & \(\neg\)REJ. & 0.390 & 0.927 \\[0.3em]
& &  & $\text{CE}_2$     
     & 0.647 & REJ.         & 0.000 & 0.347
     & 0.647 & REJ.         & 0.000 & 0.346 \\[0.3em]

\cmidrule(lr){3-12}
& & \multirow{2}{*}{DN($0.2$)} 
   & $\text{CE}_{KL}$ 
      & 0.621 & REJ.         & 0.030 & 0.999
     & 0.620 & REJ.         & 0.000 & 1.006 \\[0.3em]
& &  & $\text{CE}_2$  
     & 0.621 & REJ.         & 0.020 & 0.343
     & 0.620 & REJ.         & 0.020 & 0.343 \\[0.3em]

\cmidrule(lr){3-12}
& & \multirow{2}{*}{DN($0.5$)} 
   & $\text{CE}_{KL}$   
    & 0.588 & \(\neg\)REJ. & 0.210 & 1.282
     & 0.587 & \(\neg\)REJ. & 0.410 & 1.274 \\[0.3em]
& &  & $\text{CE}_2$     
      & 0.588 & REJ.         & 0.010 & 0.417
     & 0.587 & REJ.         & 0.010 & 0.416 \\[0.3em]
\bottomrule
\end{tabular}}
\caption{Calibration test results on CIFAR-$10$ and CIFAR-$100$ for the significance level $\alpha=0.05$ and $D=100$ bootstrap iterations. 
For each architecture (ResNet-$18$ or VGG-$19$) and model type 
(Deep Ensemble (DE) or MC Dropout (DN($p$)) with rate $p \in \{0.2, 0.5\}$, we show results of running the proposed test using two different calibration errors, $g\in \{\text{CE}_{KL}, \text{CE}_2\}$. For $g=\text{CE}_{KL}$, log loss, for $g=\text{CE}_2$ the Brier score is used as the proper scoring rule for optimization.
Columns 4--7 show accuracy, test decision, 
\(p\)--value, and value of the test statistic for the average ensemble combination 
(\(\boldsymbol{\lambda}_{\mathrm{AVG}}\)), and columns 8--11 for the learned convex combination 
(\(\boldsymbol{\lambda}^*\)).}
\label{tab: results cifar10 and cifar100}
\end{table}

\label{sec: experiments real data}
Since the true data-generating distribution is unknown in real-world datasets, we cannot directly quantify Type $1$ or Type $2$ errors. Instead, we demonstrate the practical usefulness of our test by applying it to the two standard image classification tasks CIFAR-$10$ and CIFAR-$100$ \cite{krizhevsky2009learning}. More specifically, we apply our test on the predictions of three ensemble methods, combined with two different architectures that are trained on the two datasets, respectively. We train a deep ensemble (DE, \cite{NIPS2017_lakshminarayan}), a dropout network with dropout rate $0.5$ (DN($0.5$), \cite{gal2016dropout}) and a dropout network with dropout rate $0.2$ (DN($0.2$)). For the deep ensemble, we train $10$ different models using different weight initializations, while in the dropout networks, dropout is applied after specific layers.  \added{A higher dropout rate injects more parameter noise during training, increasing predictor diversity and thus enlarging the credal set, which is why we analyse the two different cases}. 
After training, we obtain a set of different predictions from these models, which we then use to run Algorithm \ref{alg: Alg v2.0}.
As model architectures, we use two architectures of different model complexity: the ResNet-$18$ \cite{he2016deep} ($11.6 \times 10^6$ parameters) and the VGG-$19$ \cite{simonyan2014very} ($138 \times 10^6$ parameters). For the optimisation part, we leverage the same architecture, i.e.\ a pre-trained ResNet-18 or VGG-19, to embed images and attach a fully-connected layer with $32$ neurons to predict the weight function $\boldsymbol{\lambda}$. Section \ref{sec: experimental setup} describes the full experimental setup. 
The results are shown in Table \ref{tab: results cifar10 and cifar100}. For each model and calibration estimator, we run the respective bootstrapping test (line 5-19 of Algorithm \ref{alg: Alg v2.0}), with the optimal weight function $\boldsymbol{\lambda}^*$ learned by the neural network. Similarly as in the synthetic experiments, we also perform the bootstrapping part of the test with the mean predictor (Algorithm \ref{alg: Alg v2.0} from line $5$ with $f_{\boldsymbol{\lambda}_{AVG.}}=\frac{1}{M}\sum_{j=1}^M f^{(j)}$). The results differ significantly, depending on the used calibration error. In general, using the calibration error $\text{CE}_{KL}$ led to the fewest rejections (i.e.\, the highest p-values), potentially because the models were trained with a log-loss objective as proper scoring rule, which aligns more closely with KL-based calibration measures. In general, for the CIFAR-$100$ dataset with more classes, the test rejects more often, and the dropout networks are less calibrated. In most cases, the learned convex combination $\boldsymbol{\lambda}^*$ leads to a slight increase in accuracy for the combined prediction $f_{\boldsymbol{\lambda}^*}$.

\section{Discussion}
We developed a novel statistical test for the calibration of epistemic uncertainty representations based on sets of probability distributions. It is defined on an instance-based level, thereby making the test more flexible. As calibration forms a necessary condition for \textit{validity}, we claim that by testing for it, one can safely detect scenarios where the credal set is \textit{not} valid. For this case, the next step to include would be \textit{actionability}, i.e., ways to increase the size of the credal set to include at least one calibrated convex combination.
Here, we model calibration as a property that applies across the entire instance space. Alternatively, calibration could be viewed as an instance-specific concept, allowing analysis in different regions of the instance space. However, there has been limited research on this form of calibration, sometimes referred to as conditional or \textit{local} calibration \cite{luo2022local}, and the challenge of partitioning the instance space in a way that enables the computation of expectations remains unresolved. \added{Interesting future work includes the application of the test to other credal set representations such as interval-neural networks \cite{wang2025creinns}, credal sets based on relative log-likelihood \cite{lohr2025credal} and bootstrapped ensembles, as well as the analysis of the regression case.}
\subsection*{Acknowledgements}
 M.J. and W.W. received funding from the Flemish Government under the “Onderzoeksprogramma Artifici\"ele Intelligentie (AI) Vlaanderen” programme.









%
%
%
\bibliography{sn-bibliography}

\begin{thebibliography}{55}
\providecommand{\natexlab}[1]{#1}
\providecommand{\url}[1]{\texttt{#1}}
\expandafter\ifx\csname urlstyle\endcsname\relax
  \providecommand{\doi}[1]{doi: #1}\else
  \providecommand{\doi}{doi: \begingroup \urlstyle{rm}\Url}\fi

\bibitem[Abe et~al.(2022)Abe, Buchanan, Pleiss, Zemel, and
  Cunningham]{abe2022deep}
Taiga Abe, Estefany~Kelly Buchanan, Geoff Pleiss, Richard Zemel, and John~P
  Cunningham.
\newblock Deep ensembles work, but are they necessary?
\newblock In \emph{Proc.\ {NeurIPS}, 35th Advances in Neural Information
  Processing Systems}, 2022.

\bibitem[Bengs et~al.(2022)Bengs, H\"{u}llermeier, and
  Waegeman]{bengs2022neurips}
Viktor Bengs, Eyke H\"{u}llermeier, and Willem Waegeman.
\newblock Pitfalls of epistemic uncertainty quantification through loss
  minimisation.
\newblock In \emph{Proc.\ {NeurIPS}, 35th Advances in Neural Information
  Processing Systems}, 2022.

\bibitem[Bengs et~al.(2023)Bengs, H\"{u}llermeier, and Waegeman]{bengs2023icml}
Viktor Bengs, Eyke H\"{u}llermeier, and Willem Waegeman.
\newblock On second-order scoring rules for epistemic uncertainty
  quantification.
\newblock In \emph{Proc.\ {ICML}, 40th International Conference on Machine
  Learning}, 2023.

\bibitem[{Brier}(1950)]{brier1950}
Glenn~W. {Brier}.
\newblock {Verification of Forecasts Expressed in Terms of Probability}.
\newblock \emph{Monthly Weather Review}, 78\penalty0 (1):\penalty0 1, 1950.

\bibitem[Br{\"o}cker(2009)]{brockerReliabilitySufficiencyDecomposition2009}
Jochen Br{\"o}cker.
\newblock Reliability, {{Sufficiency}}, and the {{Decomposition}} of {{Proper
  Scores}}.
\newblock \emph{Quarterly Journal of the Royal Meteorological Society},
  135:\penalty0 1512--1519, 2009.

\bibitem[Bröcker and Smith(2007)]{broeckersmith2007}
Jochen Bröcker and Leonard~A. Smith.
\newblock Increasing the reliability of reliability diagrams.
\newblock \emph{Weather and Forecasting}, 22, 2007.

\bibitem[Caprio et~al.(2024)Caprio, Dutta, Jang, Lin, Ivanov, Sokolsky, and
  Lee]{caprio2024credal}
Michele Caprio, Souradeep Dutta, Kuk~Jin Jang, Vivian Lin, Radoslav Ivanov,
  Oleg Sokolsky, and Insup Lee.
\newblock Credal bayesian deep learning.
\newblock \emph{Transactions on Machine Learning Research}, 2024.

\bibitem[Chau et~al.(2024)Chau, Schrab, Gretton, Sejdinovic, and
  Muandet]{chau2024credal}
Siu~Lun Chau, Antonin Schrab, Arthur Gretton, Dino Sejdinovic, and Krikamol
  Muandet.
\newblock Credal two-sample tests of epistemic ignorance.
\newblock \emph{arXiv preprint arXiv:2410.12921}, 2024.

\bibitem[Cozman(2000)]{cozman2000credal}
Fabio~G Cozman.
\newblock Credal networks.
\newblock \emph{Artificial intelligence}, 120\penalty0 (2):\penalty0 199--233,
  2000.

\bibitem[D{\"u}mbgen(2017)]{duembgen2017empirische}
L.~D{\"u}mbgen.
\newblock \emph{Empirische Prozesse}.
\newblock Institut f{\"u}r Mathematische Statistik und Versicherungsmathematik
  der Universit{\"a}t Bern, Bern, Switzerland, 2017.
\newblock URL \url{https://books.google.de/books?id=tmFMygEACAAJ}.

\bibitem[Gal and Ghahramani(2016)]{gal2016dropout}
Yarin Gal and Zoubin Ghahramani.
\newblock Dropout as a bayesian approximation: Representing model uncertainty
  in deep learning.
\newblock In \emph{Proc.\ {ICML}, 32nd International Conference on Machine
  Learning}, 2016.

\bibitem[Gelman et~al.(2004)Gelman, Carlin, Stern, and Rubin]{gelmanbda04}
Andrew Gelman, John~B. Carlin, Hal~S. Stern, and Donald~B. Rubin.
\newblock \emph{Bayesian Data Analysis}.
\newblock 2nd ed. edition, 2004.

\bibitem[Gneiting and Raftery(2007)]{gneiting2007strictly}
Tilmann Gneiting and Adrian~E Raftery.
\newblock Strictly proper scoring rules, prediction, and estimation.
\newblock \emph{Journal of the American statistical Association}, 102, 2007.

\bibitem[Gruber et~al.(2023)Gruber, Schenk, Schierholz, Kreuter, and
  Kauermann]{gruber2023sources}
Cornelia Gruber, Patrick~Oliver Schenk, Malte Schierholz, Frauke Kreuter, and
  G{\"o}ran Kauermann.
\newblock Sources of uncertainty in machine learning--a statisticians' view.
\newblock \emph{arXiv preprint arXiv:2305.16703}, 2023.

\bibitem[Gruber and Buettner(2022)]{gruber2022betteruncertaintycalibration}
Sebastian Gruber and Florian Buettner.
\newblock Better uncertainty calibration via proper scores for classification
  and beyond.
\newblock In \emph{Proc.\ {NeurIPS}, 35th Advances in Neural Information
  Processing Systems}, 2022.

\bibitem[Guo et~al.(2017)Guo, Pleiss, Sun, and Weinberger]{guo2017}
Chuan Guo, Geoff Pleiss, Yu~Sun, and Kilian~Q. Weinberger.
\newblock On calibration of modern neural networks.
\newblock In \emph{Proc.\ {ICML}, 34th International Conference on Machine
  Learning}, 2017.

\bibitem[Gupta et~al.(2020)Gupta, Rahimi, Ajanthan, Mensink, Sminchisescu, and
  Hartley]{guptacalibration}
Kartik Gupta, Amir Rahimi, Thalaiyasingam Ajanthan, Thomas Mensink, Cristian
  Sminchisescu, and Richard Hartley.
\newblock Calibration of neural networks using splines.
\newblock In \emph{Proc.\ {ICML}, International Conference on Learning
  Representations}, 2020.

\bibitem[He et~al.(2016)He, Zhang, Ren, and Sun]{he2016deep}
Kaiming He, Xiangyu Zhang, Shaoqing Ren, and Jian Sun.
\newblock Deep residual learning for image recognition.
\newblock In \emph{Proceedings of the IEEE conference on computer vision and
  pattern recognition}, pages 770--778, 2016.

\bibitem[Hosmer et~al.(1997)Hosmer, Hosmer, Le~Cessie, and
  Lemeshow]{hosmer1997comparison}
David~W Hosmer, Trina Hosmer, Saskia Le~Cessie, and Stanley Lemeshow.
\newblock A comparison of goodness-of-fit tests for the logistic regression
  model.
\newblock \emph{Statistics in medicine}, 16, 1997.

\bibitem[H{\"u}llermeier and
  Waegeman(2021)]{hullermeierAleatoricEpistemicUncertainty2019}
Eyke H{\"u}llermeier and Willem Waegeman.
\newblock Aleatoric and epistemic uncertainty in machine learning: An
  introduction to concepts and methods.
\newblock \emph{Machine learning}, 110, 2021.

\bibitem[H{\"u}llermeier et~al.(2022)H{\"u}llermeier, Destercke, and
  Shaker]{hullermeier2022credalunc}
Eyke H{\"u}llermeier, S{\'e}bastien Destercke, and Mohammad~Hossein Shaker.
\newblock Quantification of credal uncertainty in machine learning: A critical
  analysis and empirical comparison.
\newblock In \emph{Proc.\ {UAI}, 34th Conference on Uncertainty in Artificial
  Intelligence}, 2022.

\bibitem[Javanmardi et~al.(2024)Javanmardi, Stutz, and
  H\"{u}llermeier]{javanmardi2024conformalized}
Alireza Javanmardi, David Stutz, and Eyke H\"{u}llermeier.
\newblock Conformalized credal set predictors.
\newblock In \emph{Proc.\ {NeurIPS}, 37th Advances in Neural Information
  Processing Systems}, 2024.

\bibitem[Juergens et~al.(2024)Juergens, Meinert, Bengs, H{\"u}llermeier, and
  Waegeman]{juergens2024is}
Mira Juergens, Nis Meinert, Viktor Bengs, Eyke H{\"u}llermeier, and Willem
  Waegeman.
\newblock Is epistemic uncertainty faithfully represented by evidential deep
  learning methods?
\newblock In \emph{Proc.\ {ICML}, 41st International Conference on Machine
  Learning}, 2024.

\bibitem[Kendall and Gal(2017)]{kendallWhatUncertaintiesWe2017}
Alex Kendall and Yarin Gal.
\newblock What {{Uncertainties Do We Need}} in {{Bayesian Deep Learning}} for
  {{Computer Vision}}?
\newblock In \emph{Proc.\ {NeurIPS}, 30th Advances in Neural Information
  Processing Systems}, 2017.

\bibitem[Kopetzki et~al.(2021)Kopetzki, Charpentier, Z{\"u}gner, Giri, and
  G{\"u}nnemann]{kopetzki2021evaluating}
Anna-Kathrin Kopetzki, Bertrand Charpentier, Daniel Z{\"u}gner, Sandhya Giri,
  and Stephan G{\"u}nnemann.
\newblock Evaluating robustness of predictive uncertainty estimation: Are
  dirichlet-based models reliable?
\newblock In \emph{Proc.\ {ICML}, 38th International Conference on Machine
  Learning}, 2021.

\bibitem[Krizhevsky et~al.(2009)Krizhevsky, Hinton,
  et~al.]{krizhevsky2009learning}
Alex Krizhevsky, Geoffrey Hinton, et~al.
\newblock Learning multiple layers of features from tiny images., 2009.

\bibitem[Kull and Flach(2015)]{kull2015noveldecompositions}
Meelis Kull and Peter Flach.
\newblock Novel decompositions of proper scoring rules for classification:
  Score adjustment as precursor to calibration.
\newblock In \emph{Machine Learning and Knowledge Discovery in Databases},
  2015.

\bibitem[Kumar et~al.(2018)Kumar, Sarawagi, and Jain]{pmlr-v80-kumar18a}
Aviral Kumar, Sunita Sarawagi, and Ujjwal Jain.
\newblock Trainable calibration measures for neural networks from kernel mean
  embeddings.
\newblock In \emph{Proc.\ {ICML}, 35th International Conference on Machine
  Learning}, 2018.

\bibitem[Lakshminarayanan et~al.(2017)Lakshminarayanan, Pritzel, and
  Blundell]{NIPS2017_lakshminarayan}
Balaji Lakshminarayanan, Alexander Pritzel, and Charles Blundell.
\newblock Simple and scalable predictive uncertainty estimation using deep
  ensembles.
\newblock In \emph{Proc.\ {NeurIPS}, 30th Advances in Neural Information
  Processing Systems}, 2017.

\bibitem[L{\"o}hr et~al.(2025)L{\"o}hr, Hofman, Mohr, and
  H{\"u}llermeier]{lohr2025credal}
Timo L{\"o}hr, Paul Hofman, Felix Mohr, and Eyke H{\"u}llermeier.
\newblock Credal prediction based on relative likelihood.
\newblock \emph{arXiv preprint arXiv:2505.22332}, 2025.

\bibitem[Luo et~al.(2022)Luo, Bhatnagar, Bai, Zhao, Wang, Xiong, Savarese,
  Ermon, Schmerling, and Pavone]{luo2022local}
Rachel Luo, Aadyot Bhatnagar, Yu~Bai, Shengjia Zhao, Huan Wang, Caiming Xiong,
  Silvio Savarese, Stefano Ermon, Edward Schmerling, and Marco Pavone.
\newblock Local calibration: metrics and recalibration.
\newblock In \emph{Proc.\ {UAI}, Uncertainty in Artificial Intelligence}, 2022.

\bibitem[Marx et~al.(2024)Marx, Zalouk, and Ermon]{marx2024calibration}
Charlie Marx, Sofian Zalouk, and Stefano Ermon.
\newblock Calibration by distribution matching: Trainable kernel calibration
  metrics.
\newblock \emph{Proc.\ {NeurIPS}, 36th Advances in Neural Information
  Processing Systems}, 2024.

\bibitem[Meinert et~al.(2023)Meinert, Gawlikowski, and
  Lavin]{meinert2023unreasonable}
Nis Meinert, Jakob Gawlikowski, and Alexander Lavin.
\newblock The unreasonable effectiveness of deep evidential regression.
\newblock In \emph{Proc.\ {AAAI}, 37th Proceedings of the AAAI Conference on
  Artificial Intelligence}, 2023.

\bibitem[Mortier et~al.(2023)Mortier, Bengs, H\"ullermeier, Luca, and
  Waegeman]{mortier2023calibration}
Thomas Mortier, Viktor Bengs, Eyke H\"ullermeier, Stijn Luca, and Willem
  Waegeman.
\newblock On the calibration of probabilistic classifier sets.
\newblock In \emph{Proceedings of The 26th International Conference on
  Artificial Intelligence and Statistics}, 2023.

\bibitem[M{\"u}ller(1997)]{muller1997integral}
Alfred M{\"u}ller.
\newblock Integral probability metrics and their generating classes of
  functions.
\newblock \emph{Advances in applied probability}, 29, 1997.

\bibitem[Murphy(1973)]{murphy1973brierdecomposition}
Allan~H. Murphy.
\newblock A new vector partition of the probability score.
\newblock \emph{Journal of Applied Meteorology (1962-1982)}, 12\penalty0
  (4):\penalty0 595--600, 1973.

\bibitem[Naeini et~al.(2015)Naeini, Cooper, and Hauskrecht]{naeini2015}
Mahdi~Pakdaman Naeini, Gregory~F. Cooper, and Milos Hauskrecht.
\newblock Obtaining well calibrated probabilities using bayesian binning.
\newblock In \emph{Proc.\ {AAAI},Proceedings of the Twenty-Ninth AAAI
  Conference on Artificial Intelligence}, 2015.

\bibitem[Nguyen et~al.(2022)Nguyen, Shaker, and
  H{\"u}llermeier]{nguyen2022measure}
Vu-Linh Nguyen, Mohammad~Hossein Shaker, and Eyke H{\"u}llermeier.
\newblock How to measure uncertainty in uncertainty sampling for active
  learning.
\newblock \emph{Machine Learning}, 111, 2022.

\bibitem[Niculescu-Mizil and Caruana(2005)]{niculescu2005predicting}
Alexandru Niculescu-Mizil and Rich Caruana.
\newblock Predicting good probabilities with supervised learning.
\newblock In \emph{Proc.\ {ICML}, 22nd international conference on Machine
  learning}, 2005.

\bibitem[Ovadia et~al.(2019)Ovadia, Fertig, Ren, Nado, Sculley, Nowozin,
  Dillon, Lakshminarayanan, and Snoek]{ovadia2019nips}
Yaniv Ovadia, Emily Fertig, Jie Ren, Zachary Nado, D.~Sculley, Sebastian
  Nowozin, Joshua Dillon, Balaji Lakshminarayanan, and Jasper Snoek.
\newblock Can you trust your model\textquotesingle s uncertainty? evaluating
  predictive uncertainty under dataset shift.
\newblock In \emph{Proc.\ {NeurIPS}, 31st Advances in Neural Information
  Processing Systems}, 2019.

\bibitem[Popordanoska et~al.(2022)Popordanoska, Sayer, and
  Blaschko]{popordanoska2022consistent}
Teodora Popordanoska, Raphael Sayer, and Matthew Blaschko.
\newblock A consistent and differentiable lp canonical calibration error
  estimator.
\newblock 2022.

\bibitem[Popordanoska et~al.(2024)Popordanoska, Gregor~Gruber, Tiulpin,
  Buettner, and B.~Blaschko]{pmlr-v238-popordanoska24a}
Teodora Popordanoska, Sebastian Gregor~Gruber, Aleksei Tiulpin, Florian
  Buettner, and Matthew B.~Blaschko.
\newblock Consistent and asymptotically unbiased estimation of proper
  calibration errors.
\newblock In \emph{Proceedings of The 27th International Conference on
  Artificial Intelligence and Statistics}, 2024.

\bibitem[Rahimi et~al.(2020)Rahimi, Shaban, Cheng, Hartley, and
  Boots]{rahimi2020intra}
Amir Rahimi, Amirreza Shaban, Ching-An Cheng, Richard Hartley, and Byron Boots.
\newblock Intra order-preserving functions for calibration of multi-class
  neural networks.
\newblock \emph{Proc.\ {NeurIPS}, 33rd Advances in Neural Information
  Processing Systems}, 2020.

\bibitem[Sale et~al.(2023)Sale, Caprio, and H\"ullermeier]{sale2023volume}
Yusuf Sale, Michele Caprio, and Eyke H\"ullermeier.
\newblock Is the volume of a credal set a good measure for epistemic
  uncertainty?
\newblock In \emph{Proc.\ {UAI}, 39th Conference on Uncertainty in Artificial
  Intelligence}, 2023.

\bibitem[Senge et~al.(2014)Senge, B{\"o}sner, Dembczy{\'n}ski, Haasenritter,
  Hirsch, Donner-Banzhoff, and H{\"u}llermeier]{SENGE201416}
Robin Senge, Stefan B{\"o}sner, Krzysztof Dembczy{\'n}ski, J{\"o}rg
  Haasenritter, Oliver Hirsch, Norbert Donner-Banzhoff, and Eyke
  H{\"u}llermeier.
\newblock Reliable classification: Learning classifiers that distinguish
  aleatoric and epistemic uncertainty.
\newblock \emph{Information Sciences}, 2014.

\bibitem[Simonyan and Zisserman(2014)]{simonyan2014very}
Karen Simonyan and Andrew Zisserman.
\newblock Very deep convolutional networks for large-scale image recognition.
\newblock \emph{arXiv preprint}, 2014.

\bibitem[Ulmer et~al.(2023)Ulmer, Hardmeier, and
  Frellsen]{ulmer2024priorposteriornetworks}
{Dennis Thomas} Ulmer, Christian Hardmeier, and Jes Frellsen.
\newblock Prior and posterior networks: A survey on evidential deep learning
  methods for uncertainty estimation.
\newblock \emph{Transactions on Machine Learning Research}, 2023.

\bibitem[Vaicenavicius et~al.(2019)Vaicenavicius, Widmann, Andersson, Lindsten,
  Roll, and Sch\"{o}n]{vaicenavicius19a}
Juozas Vaicenavicius, David Widmann, Carl Andersson, Fredrik Lindsten, Jacob
  Roll, and Thomas Sch\"{o}n.
\newblock Evaluating model calibration in classification.
\newblock In \emph{Proceedings of the Twenty-Second International Conference on
  Artificial Intelligence and Statistics}, 2019.

\bibitem[van~der Vaart(2000)]{van2000asymptotic}
A.~W. van~der Vaart.
\newblock Asymptotic statistics.
\newblock \emph{Cambridge University Press}, 2000.

\bibitem[Walley(1991)]{walley1991}
Peter Walley.
\newblock Statistical reasoning with imprecise probabilities.
\newblock 1991.

\bibitem[Wang et~al.(2025)Wang, Shariatmadar, Manchingal, Cuzzolin, Moens, and
  Hallez]{wang2025creinns}
Kaizheng Wang, Keivan Shariatmadar, Shireen~Kudukkil Manchingal, Fabio
  Cuzzolin, David Moens, and Hans Hallez.
\newblock Creinns: Credal-set interval neural networks for uncertainty
  estimation in classification tasks.
\newblock \emph{Neural Networks}, 2025.

\bibitem[Widmann et~al.(2019)Widmann, Lindsten, and
  Zachariah]{widmann2019calibration}
David Widmann, Fredrik Lindsten, and Dave Zachariah.
\newblock Calibration tests in multi-class classification: A unifying
  framework.
\newblock 2019.

\bibitem[Wolpert(1992)]{wolpert1992stacked}
David~H Wolpert.
\newblock Stacked generalization.
\newblock \emph{Neural networks}, 5, 1992.

\bibitem[Zadrozny and Elkan(2001)]{zadrozny2001}
Bianca Zadrozny and Charles Elkan.
\newblock Obtaining calibrated probability estimates from decision trees and
  naive bayesian classifiers.
\newblock In \emph{Proc.\ {ICML}, 18th International Conference on Machine
  Learning}, 2001.

\bibitem[Zadrozny and Elkan(2002)]{zadrozny2002transforming}
Bianca Zadrozny and Charles Elkan.
\newblock Transforming classifier scores into accurate multiclass probability
  estimates.
\newblock In \emph{Proc. ACM SIGKDD, 8th international conference on Knowledge
  discovery and data mining}, 2002.

\end{thebibliography}
\begin{appendices}
\section{Estimator analysis}\label{sec: estimator analysis}
In this section we empirically analyse the distribution of the chosen calibration estimators and their behaviour with respect to the distance from the underlying ground truth probability distribution.
Figure \ref{fig: histogram estimators h0} shows the distribution of the estimators under the null hypothesis.
In the figure, we see that the mean of the estimators $\widehat{\text{CE}}_{2}$ and $\widehat{\text{CE}}_{KL}$ is slightly above zero - since they are only asymptotically unbiased - while 
$\widehat{\text{CE}}_{MMD}$ and $\widehat{\text{CE}}_{k}$ also empirically show their unbiasedness. The kernel calibration estimator $\widehat{CE}_k$ is symmetrically distributed around zero -- its asymptotic normality was also shown by Widmann et al.\ \cite{widmann2019calibration}. As the estimators $\widehat{CE}_k$ and $\widehat{CE}_{MMD}$ can obtain negative values, we use their squared versions for the optimization in Eq.\ (\ref{eq: combined calibration loss}).
\begin{figure}[h]
    \centering
    \includegraphics[width=0.7\linewidth]{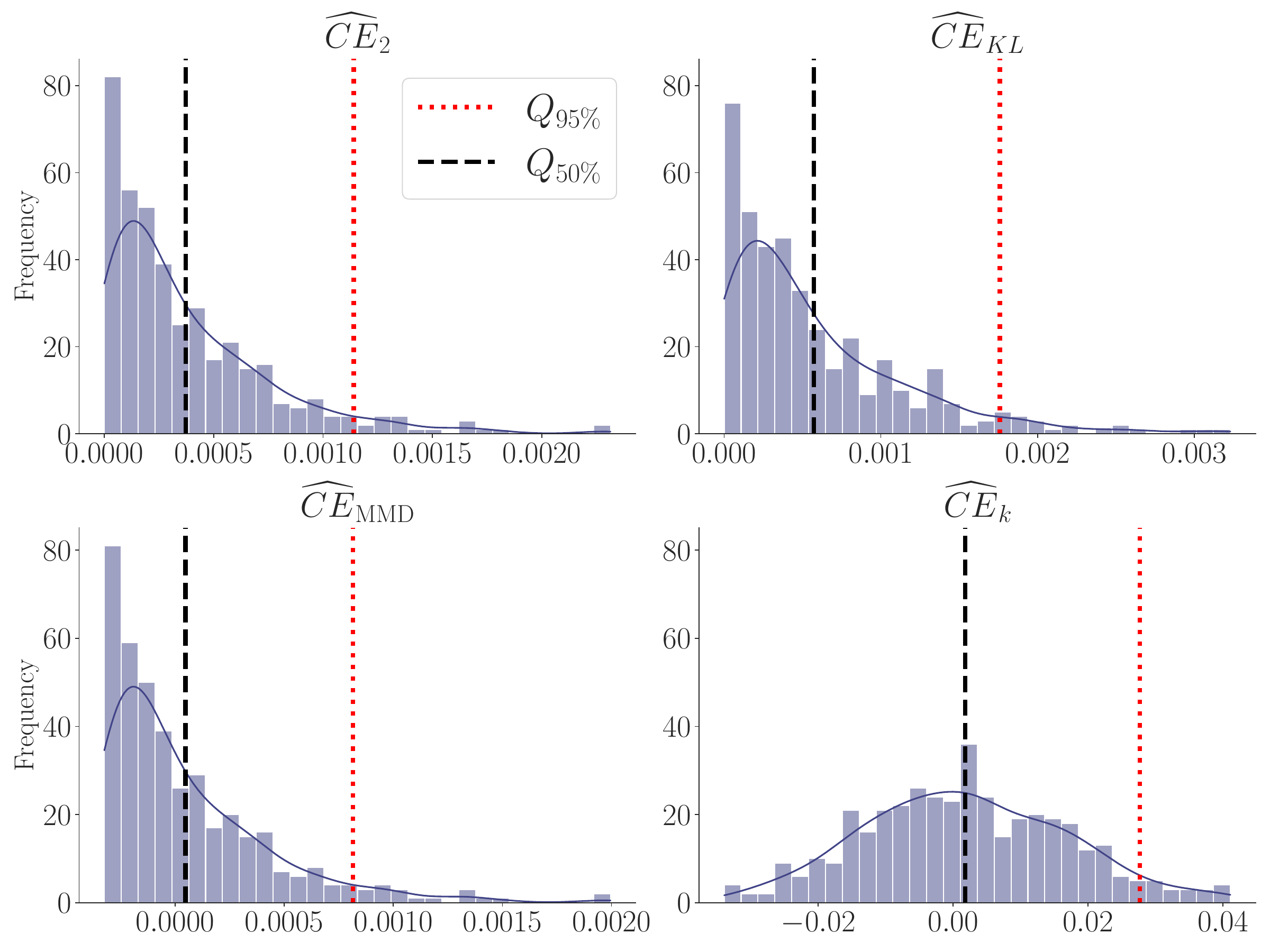}
    \caption{\textit{Distributions of calibration error estimators under the null hypothesis}. Histograms (with kernel density estimates) of the estimated calibration error $ g \in \{\widehat{CE}_{2},\widehat{CE}_{KL}, \widehat{CE}_{\mathrm{MMD}},\widehat{CE}_{k} \}$ are shown. Predicted probabilities were sampled from a uniform Dirichlet distribution over three classes ($K = 3$) and repeated $N = 1000$ times to simulate a perfectly calibrated model. Labels were sampled from the corresponding categorical distribution in $D = 500$ resampling steps. Vertical dashed black lines indicate the mean values of each statistic. Vertical dashed red lines indicate the $95\%$ quantile of the empirical distribution. }
    \label{fig: histogram estimators h0}
\end{figure}

Figure \ref{fig: heatmaps calibration estimates} shows the values of different calibration estimators in dependence of the position in the simplex, where the true underlying calibrated convex combination $f_{\boldsymbol{\lambda}^*}$ is set with $\boldsymbol{\lambda}^*=(0.1, 0.1, 0.8)$, and the predictions are sampled from a Dirichlet distribution. We see that by optimizing over $\boldsymbol{\lambda}$, the true calibrated convex combination is learned differently "well enough". However, $\boldsymbol{\lambda}^*$ still lies within a certain region of low estimator values. For $\widehat{CE}_k$, we see in general more noisy behaviour in the region around the ground truth distribution (red dot).
\begin{figure}[h]
    \centering
    \includegraphics[width=0.7\linewidth]{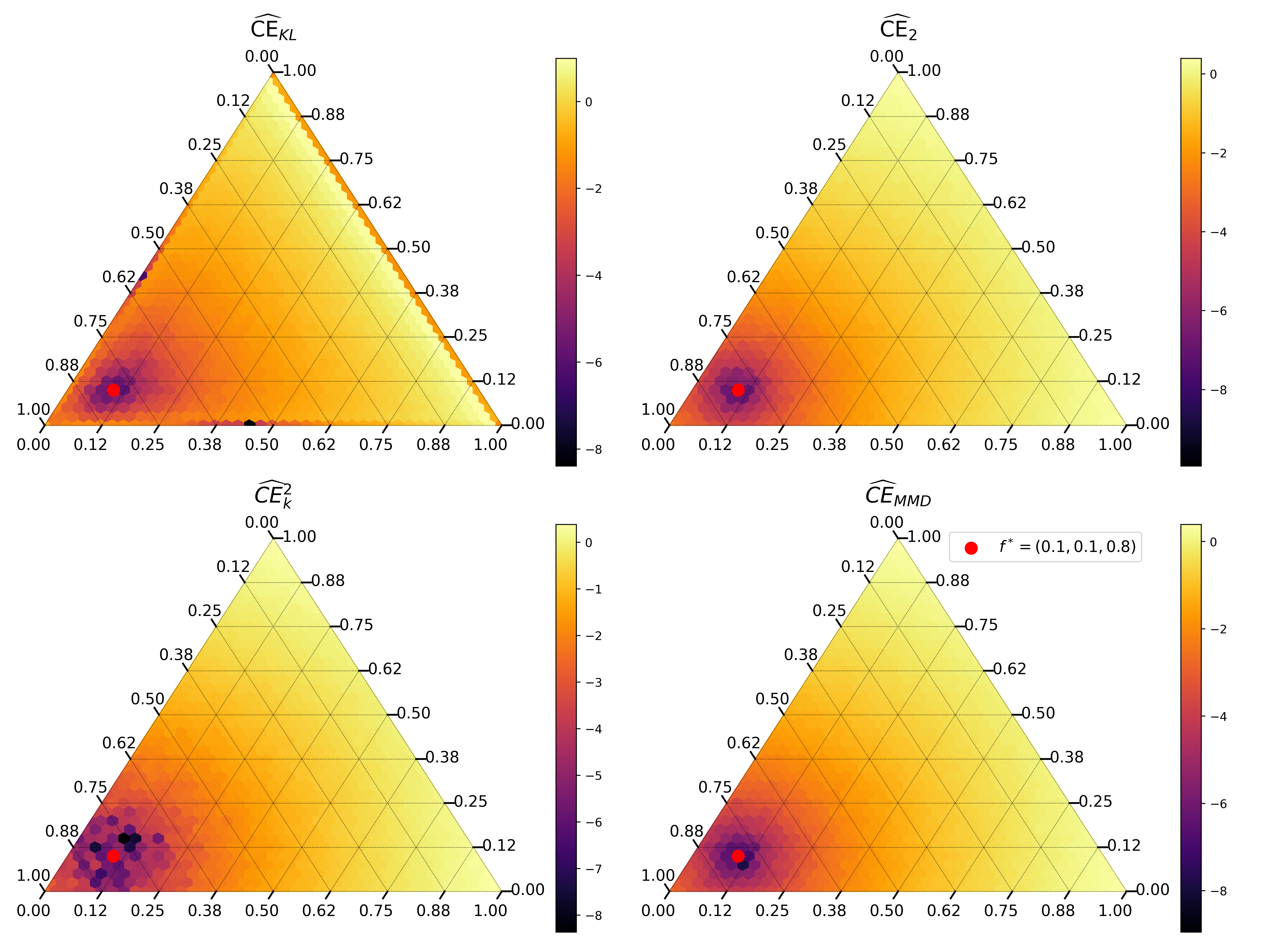}
    \caption{Simple experiment that illustrates the behaviour of the calibration estimators for the case of a \textit{non-instance-dependent} underlying ground truth distribution. Log-transformed values of the empirical calibration error $\hat{g}(\hat{f}, \mathcal{D})$ where $\hat{g} \in \{\widehat{\text{CE}}_{KL}, \widehat{\text{CE}_2}, \widehat{\text{CE}}_{k}^2, \widehat{\text{CE}}_{MMD}\},$ in dependence of the (constant) predictor $\hat{f}$ within the probability simplex, for the case of $K=3$ classes. $N=2000$ labels $y_i$ in $\mathcal{D}$ were generated from the ground truth categorical distribution, $y_i \sim \text{Cat}(f^*), \, i=1, \dots, 2000$. Here, we set $f^*(x) \equiv (0.1, 0.1, 0.8)^T$ constant for all $x \in \mathcal{X}$. The red point corresponds to $f^*$ and represents the theoretical minimum of the calibration error.}
    \label{fig: heatmaps calibration estimates}
\end{figure}

We also do a similar analysis in $\lambda$-space: In Figure \ref{fig: heatmaps lambda space}, we show the value of the calibration estimators not in distance to the ground truth probability distribution, but the ground truth weight vector $\boldsymbol{\lambda}\in \Delta_3$, given the predictions of two predictors $\{f^{(1)}(x), f^{(2)}(x), f^{(3)}(x)\}$ and a ground truth probability distribution $f_{\boldsymbol{\lambda}^*}(x) = 0.1 \cdot f^{(1)}(x) + 0.1 \cdot f^{(2)}(x) + 0.8 \cdot f^{(3)}(x)$, that is $\boldsymbol{\lambda}^*=(0.1,0,1,0.8)^T$. We see that here the heatmaps vary more significantly for the different calibration estimators. 
\begin{figure}[h]
    \centering
    \includegraphics[width=.7\linewidth]{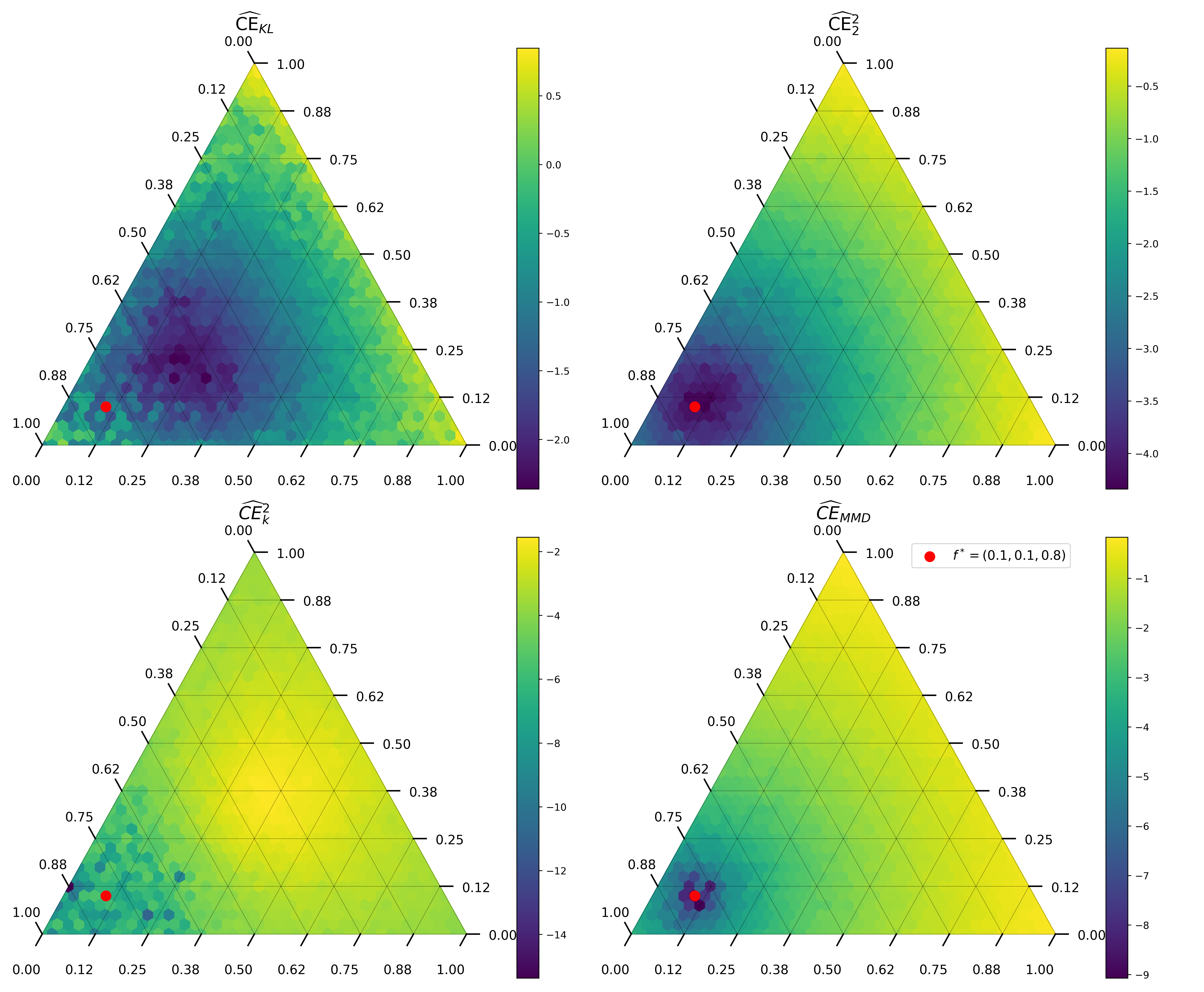}
    \caption{Heatmaps of the values of the chosen calibration estimators for a given location within the space of weights $\boldsymbol{\lambda} \in \Delta_3$. The predictions $f^{(i)} \in \{f^{(1)}, f^{(2)}, f^{(3)}\}$ are generated according to a Dirichlet distribution with parameters $ \alpha_j = 1 + \mathbb{I}_{i=j} \cdot 10$ for $j \in \{1,2,3\}$.}
    \label{fig: heatmaps lambda space}
\end{figure}
\section{Non-uniqueness of calibrated models}
\label{sec: many calibrated convex combinations}
Using simple examples, we will show the non-uniqueness of calibrated convex combinations, that is, that there are often many different (possibly instance-dependent) convex combinations leading to a calibrated predictor. Using a proposition proven by Vaicenavicius et al.\ \cite{vaicenavicius19a}, which states that conditioning $Y$ on any measurable function $h$ yields a calibrated predictor, we show this with a concrete example.
\begin{example}[Many calibrated classifiers]
\label{example: many calibrated classifiers} For the following, assume the case of having $3$ classes, i.e.\ $\mathcal{Y}= \{1,2,3\}$.
    We have (\cite{vaicenavicius19a}, Proposition 1) for any measurable function $h: \mathcal{X} \mapsto \mathcal{Z}$ with $\mathcal{Z}$ being
    some measurable space, that the function \begin{equation*}
        f(X):= \mathbb{P}[Y \in \cdot | h(X)]
    \end{equation*}
    is a calibrated classifier. 
        As an example, taking a constant $h(x) = c \in \mathbb{R} \, \forall x \in \mathcal{X}$,
    the classifier that simply predicts the marginal class probabilities \begin{equation*}
        f \equiv \begin{pmatrix}
            \mathbb{P}(Y=1) \\
            \mathbb{P}(Y=2) \\
            \mathbb{P}(Y=3) \\
        \end{pmatrix}
    \end{equation*}
    is calibrated for $\mathcal{Y}= \{1, 2, 3 \}$. Further, if we take $h: \mathcal{X}\mapsto \mathcal{X}$ with $h(x)=x$, then \begin{equation*}
        f' \equiv \begin{pmatrix}
            \mathbb{P}(Y=1|X) \\
            \mathbb{P}(Y=2|X) \\
            \mathbb{P}(Y=3|X) \\
        \end{pmatrix}
    \end{equation*}
    is also a calibrated classifier. To make this example more concrete, let $\mathcal{X}=\{1,2,3\}$ with
    $\mathbb{P}(X=x)=\frac{1}{3} \forall \, x \in \mathcal{X}$. Further let again $\mathcal{Y}=\{1,2,3\}$ and \begin{equation*}
        \mathbb{P}(Y=i|X=j) = \begin{cases}
            1, i=j \\
            0, \text{else}
        \end{cases}
    \end{equation*}
    for $i, j \in \{1,2,3\}$.
Then  $\mathbb{P}(Y=i)= \sum_{j=1}^3 \mathbb{P}(Y=i|X=j)\mathbb{P}(X=j)= \frac{1}{3}$
for $i \in \{1,2,3\}$, and \begin{equation*}
    f(x) \equiv \begin{pmatrix}
        1/3 \\
        1/3 \\
        1/3
    \end{pmatrix},
\quad f'(x) = \mathbf{e}_{x}
\end{equation*}
where $\mathbf{e}_{x}$ denotes the $x$-th unit vector, are both calibrated predictors.
\end{example}
Going from the case of one single classifier, similarly, one can also show that there usually exists \textit{various calibrated convex combinations} for a set of classifiers. The following example illustrates this for the simple case of binary classification.
\begin{example}
    Let us look at a very simple problem of binary classification with $\mathcal{Y} = \{1, 2 \}$ and $\mathcal{X}= \{-1,1\}$. Assume \begin{equation*}
        \mathbb{P}(Y=i|X=x_j)= \begin{cases}

            1, \quad (i, x_j) \in \{(1, -1), (2,1)\} \\
            0, \quad \text{else},
        \end{cases}
    \end{equation*}
 and $\mathbb{P}(X=x) = \frac{1}{2}, \quad x \in \mathcal{X}$. Then the two classifiers \begin{align*}
        f_1: \, & \mathcal{X} \rightarrow \Delta^2 \\
        & x \mapsto \begin{pmatrix}
            \frac{1}{2} \\
            \frac{1}{2}
        \end{pmatrix}, \quad x \in \{-1,1\}
    \end{align*}
    and \begin{align*}
        f_2: \, & \mathcal{X} \rightarrow \Delta^2 \\
        & x \mapsto \begin{cases}
            \begin{pmatrix}
                1 \\
                0 
            \end{pmatrix}, \quad x = -1 \\
        \begin{pmatrix}
            0 \\
            1
        \end{pmatrix}, \quad x = 1
        \end{cases}
    \end{align*}
    are both calibrated (which can easily be seen by Proposition 1 of \cite{vaicenavicius19a} and 
    conditioning on $h_1(x)=c$ and $h_2(x)=x$). Further, they can be written as convex combinations of the two classifiers $m_1(x) = \begin{pmatrix}
        0 \\
        1
    \end{pmatrix}, \forall x$ and $m_2(x) = \begin{pmatrix}
        1 \\
        0
    \end{pmatrix} \forall x$ using the two different convex combinations \begin{align*}
        \lambda_1:\, & \mathcal{X} \rightarrow \Delta^2 \\
        & x \mapsto \begin{cases}
            \begin{pmatrix}
                0 \\
                1
            \end{pmatrix}, x = -1 \\
            \begin{pmatrix}
                1 \\
                0
            \end{pmatrix}, x = 1
        \end{cases}
    \end{align*}
    and \begin{align*}
        \lambda_2: \, & \mathcal{X} \rightarrow \Delta^2 \\
        & x \mapsto \begin{pmatrix}
            \frac{1}{2} \\
            \frac{1}{2}
        \end{pmatrix}, x \in \{-1,1\}.
    \end{align*}
    \end{example}
\section{Validity of the proposed algorithm}
\label{sec: validity}
\added{In the following, we prove the asymptotic validity of the derived test.}
\begin{theorem}[Asymptotic validity of Algorithm \ref{alg: Alg v2.0}]
    Let $\mathcal{D}_{opt} =\{(x_i, y_i)\}_{i=1}^{n}$ and $\mathcal{D}_{val} = \{(x_i', y_i')\}_{i=1}^{m}$ be i.i.d. data generated from the underlying distribution $\mathbb{P}_{X,Y}$. Denote by \begin{equation*}
        \hat{\boldsymbol\lambda}_{n}\;\in\;
        \arg\min_{\Lambda}
        \frac{1}{n} \; \sum_{i=1}^n \ell(f_{\Lambda}(x_i),y_i)
                 +\gamma\,\widehat g_{n}(f_{\Lambda},\mathcal{D}_{opt})
    \end{equation*}
    the weight function that is the risk minimizer of the empirical version of Eq.\ (\ref{eq: combined calibration loss}) evaluated at $\mathcal{D}_{opt}$, where $\ell$ is the (strictly) proper Brier score or log loss and $\widehat{g}_n$ is a consistent, differentiable and (asymptotically) unbiased estimator of one of the calibration errors in Section \ref{sec: calibration errors}. Define $t_{m,n} := \hat{g}_m(f_{\hat{\boldsymbol{\lambda}}_n}, \mathcal{D}_{val})$ the test statistic evaluated on the combined predictor, and
    let $c_{m,D,\alpha}=\hat F^{-1}_{m,D}(1-\alpha)$ be the $(1-\alpha)$-quantile of the empirical cumulative distribution function $\hat{F}_{m,D}$
    of the bootstrap replicates $t_{0,1},\dots,t_{0,D}$ generated by the
consistency–resampling scheme of Vaicenavicius et al.\cite{vaicenavicius19a} (line $5-19$ of Algorithm \ref{alg: Alg v2.0}). 
\added{Furthermore, assume that the set of predictors $\mathcal{F}=\{f^{(1)}, \dots, f^{(M)}\}$ is such that the mapping $\boldsymbol{\lambda} \mapsto f_{\boldsymbol{\lambda}}$ is continuous over  $\mathcal{X} $ and for the respective calibration function $g$ all regularity conditions are fulfilled such that $f$ is calibrated if and only if $\text{g} = 0$.}
Then Algorithm \ref{alg: Alg v2.0} rejects $H_0$ whenever
$t_{m,n}\ge c_{m,D,\alpha}$.
Then, under $H_0$,
\begin{eqnarray*}
   \lim_{m,n\to\infty\atop D\to\infty}
   \mathbb{P}(
        t_{m,n}\ge c_{m,D,\alpha}|H_0
   )
   \;=\;\alpha \,.
\end{eqnarray*}
\end{theorem}
\begin{proof}
In the following, we assume that the null hypothesis is true, i.e., there exists (at least) one underlying function $\boldsymbol{\lambda}$ such that the combined predictor $f_{\boldsymbol{\lambda}}$ is calibrated.\\
1. Consistency:
Because $\ell$ is strictly proper,
the population level risk in Eq.\ (\ref{eq: combined calibration loss}) is uniquely minimised by some 
$\boldsymbol\lambda^\star\!\in\!\Delta_{M,\mathcal{X}}$ with respective combined predictor $f_{\boldsymbol\lambda^\star}$ . As the expectation of $\ell$ can be decomposed into a calibration and refinement term, and $g(f) = 0$ for any calibrated predictor $f$, we have that $g(f_{\boldsymbol\lambda^\star})=0$.

By the arg-min consistency theorem (Thm. $5.7$ of \cite{van2000asymptotic})
it holds that the empirical risk minimizer converges to the population level one, i.e.\
$\|\hat{\boldsymbol\lambda}_{n}-\boldsymbol\lambda^\star\|_{L_1}
  \rightarrow0$
as $n\!\to\!\infty$, and hence we have $\|f_{\hat{\boldsymbol\lambda}_{n}}-f_{\boldsymbol\lambda^*}\|_{L_1}\rightarrow0.$

2. Asymptotic distribution:
Under the null hypothesis, for each of the calibration errors we have $g(f_{\boldsymbol{\lambda}^*})=0$.
Further, taking the asymptotic behavior of U-statistics \cite{van2000asymptotic} into account,
 we have $\sqrt{m}\,\hat{g}_m(f_{\boldsymbol{\lambda}^*}, \mathcal{D}_{val}) \rightarrow Z$, $m\rightarrow \infty$,
  for some random variable with limiting distribution $F_Z$. By Rubin's theorem (see Theorem 7.9 in \cite{duembgen2017empirische}) it holds that $\sqrt{m}\, t_{m,n} \rightarrow Z$, $m,n\rightarrow\infty$.

3. Bootstrap validity: For sake of simplicity, set $D=m$. Under $H_0$, the consistency-bootstrap (line $5$ to $19$ of Algorithm \ref{alg: Alg v2.0}, \cite{vaicenavicius19a}) guarantees that conditionally on the validation data set it holds that $\sqrt{m} \, c_{m,D,\alpha}\rightarrow F_Z^{-1}(1-\alpha).$ This is because $\sup_t  |\hat F_{m,D}^{-1}(t)-F_{t_{m,n}}^{-1}(t)|\rightarrow 0 $ (compare to Section 23.2.1 in \cite{van2000asymptotic}).

4. Control of Type $1$ error: The convergence of the bootstrap quantile
against the quantile of the limiting distribution implies
$\mathbb{I}\{t_{m,n}\geq c_{m,D,\alpha}\}\rightarrow
   \mathbb{I}\{Z\geq F_Z^{-1}(1-\alpha)\}.
$ Taking expectations yields:
$\mathbb{P}(\text{reject}| H_0) = \mathbb{P}(t_{m,n} \geq c_{m.B. \alpha}) \rightarrow 1-F_Z(F_Z^{-1}(1-\alpha))=\alpha.$
\end{proof}

\section{Uniform sampling in the credal set}
\label{sec: uniform sampling}
 \begin{figure}[t]
        \centering
        \subfloat[]{\includegraphics[width=0.25\textwidth, height= 0.16\textheight]{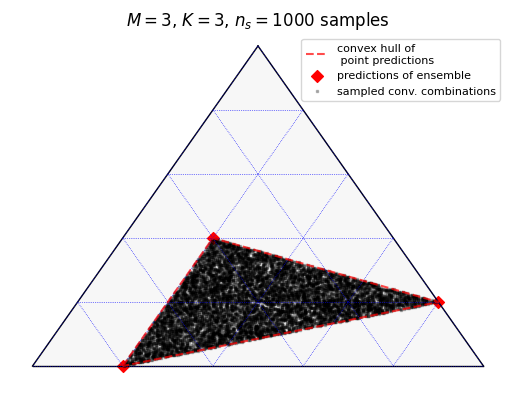}
        }
        \subfloat[]{\includegraphics[width=0.25\textwidth, height=0.16\textheight]{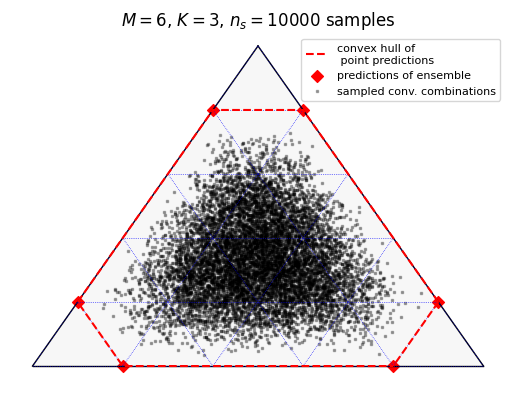}}
         \subfloat[]{\includegraphics[width=0.25\textwidth, height = 0.16\textheight]{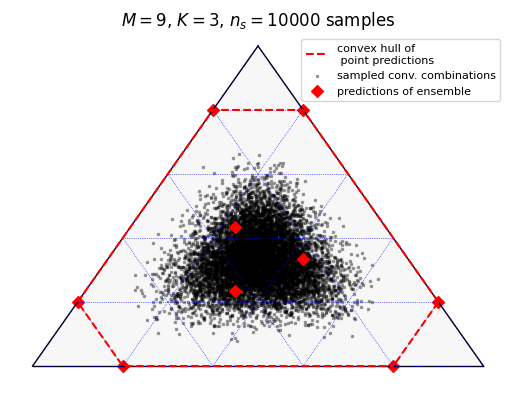}}
         \hfill
        \subfloat[]{\includegraphics[width=0.25\textwidth, height=0.16\textheight]
        {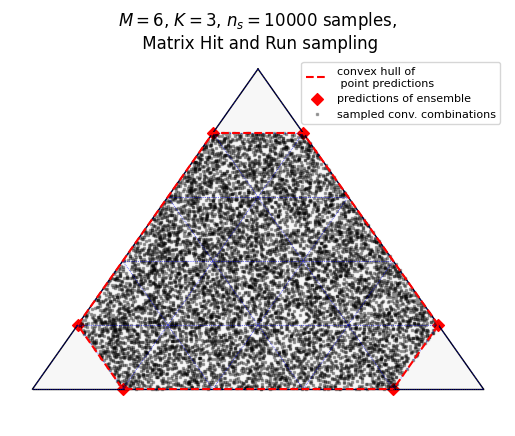}} 
        \subfloat[]{\includegraphics[width=0.25\textwidth, height=0.16\textheight]{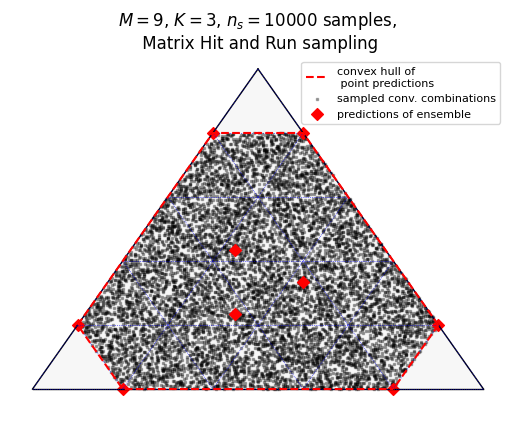}}
        \captionsetup{width=\linewidth}
        \caption{Sampling example for $K=3$: The point predictions $f^{(i)}(x)$ are visualized as red dots on the corners and/or 
        the inside of the $2$-simplex. In (a) and (b), and (c) samples $f_{\boldsymbol{\lambda}_1}, \dots, f_{\boldsymbol{\lambda}_n}$ are generated by sampling the weights of the convex combination from the uniform Dirichlet distribution, $\boldsymbol{\lambda}_i \sim Dir(1, \dots, M)$. In
        in (d) and (e), an MCMC approach is used to uniformly sample in the polytope of convex combinations.}
        \label{fig: simplices}
    \end{figure}
In the work of \citet{mortier2023calibration}, a sampling within the polytope of convex combinations of predictions is done by sampling weights $\boldsymbol{\lambda} \in \Delta_M\sim Dir(1, \dots, 1)$. This however does in general not lead to a uniform sampling within the polytope, as is exemplary shown in Figure \ref{fig: simplices}. As soon as we have more than $M=3$ predictors, the sampled predictions start "accumulating" in the centre of the simplex, due to the fact that there are more possible combinations in the centre than for the regions in the boundary. Possible ways to achieve uniform sampling is by using triangulization methods, rejection sampling or Markov Chain Monte Carlo sampling methods, however, these methods come with increased computational effort.

\section{Experimental setup}
\label{sec: experimental setup}
 In this section we explain the experimental setup for the experiments conducted in Section \ref{sec: experimental results}. 
\subsection{Synthetic data}
\begin{figure}[t]
    \centering
    \includegraphics[width=.7\textwidth]{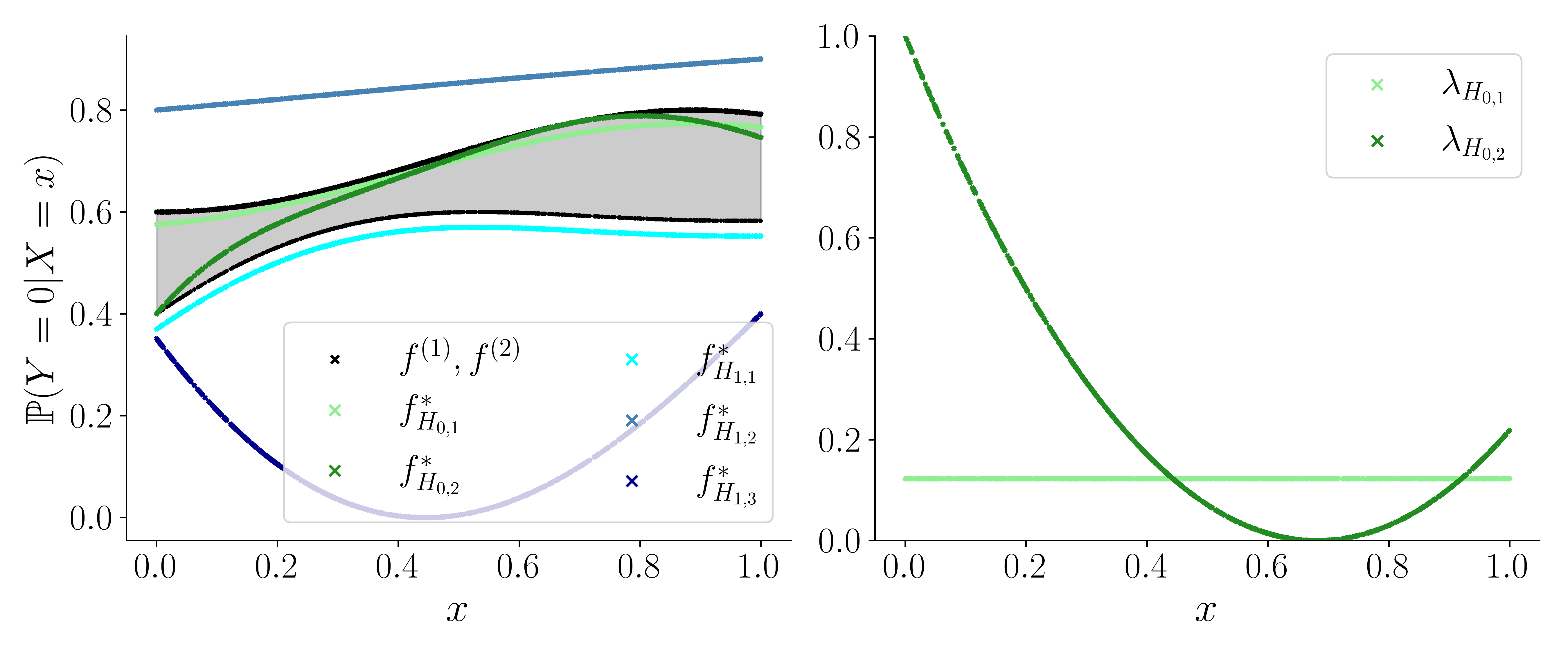}
\caption{$1$-simplex spanned by the ensemble members $\{f^{(1)}, f^{(2)}\}$ (grey) that are generated according to the binary classification setting as described in Section \ref{sec: binary classification}, and the underlying ground truth calibrated predictions  for test cases $H_{0,1}$ - $H_{1,3}$ described in Section \ref{sec: binary classification} (left). In the right graphic, the respective weights of the underlying convex combination for the cases $H_{0,1}$ (constant) and $H_{0,2}$ (instant-dependent polynomial) are visualised as a function of the instance space.
     }
    \label{fig: gp calibration cases}
\end{figure}
For optimising the weights $\boldsymbol{\lambda}$ in the synthetic experiments, we use an MLP with $3$ hidden layers each consisting of $16$ neurons that is trained on an optimization dataset of size $N_{opt}=400$. Both tests (the previous one proposed by Mortier et al.\ and ours) are performed on a validation set set of size $N_{val}=400$.\\

For the case of \textbf{binary classification}, we draw $\{x_i\}_{i=1}^N \sim U([0,5])$, and generate predictions from two probabilistic predictors $\{f^{(1)}(x), f^{(2)}(x)\}$ by sampling from a Gaussian process with an rbf kernel, whose outputs are constrained to output probabilities $[0,1]$, using min-max sclaing. Under $H_0$, we generate the probability for class $1$, $f^*(x) = \lambda^*\,f^{(1)}(x) + (1-\lambda^*)\,f^{(2)}(x)$ with $\lambda^*$ generated as \begin{itemize}
    \item $\boldsymbol{H_{0,1}}$: $\lambda^*$ is a constant and sampled as $\lambda^* \sim U(0,1)$,
    \item $\boldsymbol{H_{0,2}}$: $\lambda^*: \mathcal{X}\rightarrow [0,1]$ is a polynomial of degree $D$. In the experiments, we set $D=2$.
\end{itemize}
Under $H_1$, $f^*(x)$ lies outside the credal set with \begin{itemize}
    \item $\boldsymbol{H_{1,1}}$: $f^*(x)$ lies at an $\epsilon$-distance to one of the boundaries of the credal set. We sample $\epsilon \sim U([0,0.02])$.
    \item $\boldsymbol{H_{1,2}}-\boldsymbol{H_{1,3}}$ reflect increasing distance to the boundary by sampling a new GP that intentionally remains outside the credal set.
\end{itemize} 

For the \textbf{multi-class classification} case, we again draw $\{x_i\}_{i=1}^N \sim U([0,5])$ and generate the set of probabilistic predictors $f^{(1)}, \dots, f^{(M)}\}$ as follows: For each $x_i$, draw a prior $\boldsymbol{p}\in \mathbb{R}^K \sim Dir(1, \dots, 1)$, then sample $f^{(m)}(x_i) \sim Dir(\frac{\boldsymbol{p} \cdot K}{0.5})$. Under $H_0$, the weight function $\boldsymbol{\lambda}^*$ is generated as \begin{itemize}
    \item $\boldsymbol{H_{0,1}}$: $\boldsymbol{\lambda}^*\sim Dir(1, \dots, 1)$
    \item $\boldsymbol{H_{0,2}}$: $\boldsymbol{\lambda}^*=(\lambda_1^*, \dots, \lambda_M^*)$ with $\lambda_m^* : \mathcal{X} \rightarrow [0,1]$; the components of $\boldsymbol{\lambda}^*$ are randomly generated and scaled polynomials of degree $D=2$.
\end{itemize}
Under $H_1$, $f^c(x)$ is randomly chosen as one of the corners of the simplex, outside of the convex set. $f^*(x)$ is then given by the point prediction on the line segment between $f^c(x)$ and the boundary point $f^b(x)$ in the credal set, $f^*(x)(x) = \delta f^c(x) + (1-\delta)\,f^b(x)$, with \begin{itemize}
    \item $\boldsymbol{H_{0,1}}$: $\delta=0.01$,
    \item $\boldsymbol{H_{0,2}}$: $\delta=0.1$,
    \item $\boldsymbol{H_{0,3}}$: $\delta=0.2$.
\end{itemize}
In both settings, the labels $\{(y_i)\}_{i=1}^N$ are then sampled from the resulting categorical distribution with parameter $f^*(x_i)$.
\subsection{Real Data}
Here we provide details on the data preprocessing, model architectures, hyper-parameters, and calibration-test implementation used in the real-world experiments.\\

\noindent\textbf{Data splitting:}
We split the data into training set, (used to train the ensembles), validation and test set. The validation set together with the predictions of the models are then used to optimise for $\boldsymbol{\lambda}^*$, and the test is performed on the test data. 
\begin{itemize}
    \item CIFAR-$10$: 50000 training, 5000 validation and 5000 test samples,
    \item CIFAR-$100$: 60000 training, 5000 validation and 5000 test samples.
\end{itemize}
\textbf{Models}: We use the following training parameters to train the deep ensemble and dropout models, in order to obtain the credal set of predictions: \begin{itemize}
    \item adam optimizer with $\text{lr}=0.001$,
    \item loss function: cross-entropy (log loss)
    \item batch size: $128$
    \item number of epochs: $50$
    \item early stopping: monitor validation loss, stop if does not decrease for $\text{patience}=10$ epochs.
\end{itemize}
\textbf{Optimization of weights}: We concatenate a small MLP consisting of $3$ hidden layers and $32$ neurons to the same VGG-19/ResNet-18 architecture that has been used for learning the prediction sets. For optimization, we use the combined loss as in (\ref{eq: combined calibration loss}) with \begin{itemize}
    \item log loss and $\widehat{\text{CE}}_{KL}$ with $\gamma=0.01$ for testing $\text{CE}_{KL}$
    \item Brier score and $\widehat{\text{CE}}_2$ with $\gamma=0.01$ for testing $\text{CE}_2$.
\end{itemize}
As optimisation parameters, we use a learning rate $\text{lr}=0.0001$, a $\text{batch size} = 256$, the adam optimizer and train the neural network for $200$ epochs. 

\end{appendices}
\end{document}